\documentclass{article} % For LaTeX2e
\usepackage{amsmath,amsfonts,bm}

\def\eqref#1{equation~\ref{#1}}
\def\ceil#1{\lceil #1 \rceil}

\def\1{\bm{1}}

\def\eps{{\epsilon}}

\DeclareMathAlphabet{\mathsfit}{\encodingdefault}{\sfdefault}{m}{sl}
\SetMathAlphabet{\mathsfit}{bold}{\encodingdefault}{\sfdefault}{bx}{n}

\DeclareMathOperator*{\argmax}{arg\,max}

\usepackage{hyperref}
\usepackage{url}
\usepackage{natbib}
\usepackage{fullpage}
\usepackage[linesnumbered,vlined]{algorithm2e}
\usepackage{amsmath}
\usepackage{amsfonts, amssymb}
\usepackage{cite}
\usepackage{xcolor}
\usepackage{amsthm}
\usepackage{mathtools}
\usepackage{graphicx}
\usepackage{makecell}
\usepackage{subcaption}

\usepackage{thmtools,thm-restate}
\newtheorem{theorem}{Theorem} %[section]
\newtheorem{lemma}[theorem]{Lemma}

\usepackage{mathrsfs}

\newcounter{cntLemmaNumber}

\newcounter{cntTheoremNumber}

\newcommand{\Expc}[1]{\mathbb{E}\left[{#1}\right]}
\newcommand{\size}[1]{\ensuremath{\left|#1\right|}}
\newcommand{\norm}[1]{\ensuremath{\|#1\|}}
\newcommand{\set}[1]{\left\{ #1 \right\}}

\newcommand{\myemptyset}{{}}

\newcommand{\Gcomment}[1]{{\color{red} #1}}
\newcommand{\greedy}{\texttt{Greedy}\xspace}

\renewcommand{\cite}{\citep}

\DeclarePairedDelimiterX{\infdivx}[2]{(}{)}{%
  #1\;\delimsize\|\;#2%
}

\title{Mini-batch Submodular Maximization}
\author{Gregory Schwartzman\thanks{JAIST, \texttt{greg@jaist.ac.jp}}}
\date{}

\begin{document}

\maketitle

\begin{abstract}
We present the first \emph{mini-batch} algorithm for maximizing a non-negative monotone \emph{decomposable} submodular function, $F=\sum_{i=1}^N f^i$, under a set of constraints. 
We consider two sampling approaches: uniform and weighted. We first show that mini-batch with weighted sampling improves over the state of the art sparsifier based approach both in theory and in practice.

Surprisingly, our experimental results show that uniform sampling is superior to weighted sampling. However, it is \emph{impossible} to explain this using worst-case analysis. Our main contribution is using \emph{smoothed analysis} to provide a theoretical foundation for our experimental results. We show that, under \emph{very mild} assumptions, uniform sampling is superior for both the mini-batch and the sparsifier approaches. We empirically verify that these assumptions hold for our datasets. Uniform sampling is simple to implement and has complexity independent of $N$, making it the perfect candidate to tackle massive real-world datasets.

%To bridge this gap between theory and practice we use \emph{smoothed analysis} to provide a theoretical foundation for our experimental results.

%We improve over the sparsifier based approach both in theory and in practice. We experimentally observe that our algorithm generates solutions that are far superior to those generated by the sparsifier based approach. Surprisingly, we also observe that a simple uniform sampling approach performs as well as (and often better than) the weighted sampling approach (for \emph{both} the sparsifier and the mini-batch algorithms). These experimental results hold when combining our approach with the popular stochastic greedy algorithm. However, it is impossible to explain this using worst-case analysis. To bridge this gap between theory and practice we use \emph{smoothed analysis} to show that uniform sampling achieves almost the same guarantees as its weighted counterparts.

%We also observe that our approach is superior when combined with the popular \emph{stochastic-greedy} algorithm.
%For tiny mini-batch sizes we observe an improvement as large as 50\%-100\% on certain data-sets, making our algorithm specifically suited for huge datasets.
\end{abstract}
\section{Introduction}
Submodular functions capture the natural property of \emph{diminishing returns} which often arises in machine learning, graph theory and economics. For example, imagine you are given a large set of images, and your goal is to extract a small subset of images, that best represent the original set (e.g., creating thumbnails for a Youtube video). Intuitively, this is a submodular optimization problem, as the more thumbnails we have, the less we gain by adding an additional thumbnail.
%Furthermore, this function is decomposable as we can write $F(S) = \sum_{i=1}^N f^i(S)$, where $f^i(S)$ represents to what extent the set $S$ represent the $i$-th element of $X$. 
%One can also intuitively note that for most reasonable functions the above is monotone (e.g., more thumbnails achieve better "coverage" to the frames of the video).

Formally, a set function $F:2^E \rightarrow \mathbb{R}^+$, on a \emph{ground set} $E$, is submodular if for any subsets $S \subseteq T \subseteq E $ and $e \in E \setminus T $, it holds that $F(S+e) - F(S) \geq F(T+e) - F(T)$.
\iffalse
\begin{equation*}
\label{submodular}
F(S+e) - F(S) \geq F(T+e) - F(T)
\end{equation*}
\fi
%We usually assume that we have oracle access to $F$ and we aim to maximize $F$ under some set of constraints while minimizing the number of oracle calls to $F$. 
%One of the most common constraints is a cardinality constraint: $\max F(S), \size{S}\leq k$.

\paragraph{Decomposable submodular functions} In many natural scenarios $F$ is \emph{decomposable}: $F=\sum_{i=1}^N f^i$, where each $f^i:2^E \rightarrow \mathbb{R}^+$ is a non-negative submodular function on the ground set $E$ with $|E| = n$. That is, $F$ can be written as a sum of ``simple'' submodular functions. We assume that every $f^i$ is represented by an evaluation oracle that when queried with $S\subseteq E$ returns the value $f^i(S)$. For ease of notation we assume that $f^i(\emptyset)=0$ (our results hold even if this not the case).
Our goal is to maximize $F$ under some set of constraints while minimizing the number of oracle calls to $\set{f^i}$. An excellent survey of the importance of decomposable functions is given in \citep{rafiey2022sparsification}, which we summarize below. 

Decomposable submodular functions are prevalent in both machine learning and economics. In economics, they play a pivotal role in welfare optimization during combinatorial auctions \citep{DobzinskiS06,Feige09, FeigeV06, PapadimitriouSS08, Vondrak08}. In machine learning, these functions are instrumental in tasks like data summarization, aiming to select a concise yet representative subset of elements. Their utility spans various domains, from exemplar-based clustering by \citep{DueckF07} to image summarization \citep{TschiatschekIWB14}, recommender systems \citep{ParambathUG16} and document summarization \citep{LinB11}. The optimization of these functions, especially under specific constraints (e.g., cardinality, matroid) has been studied in various data summarization settings \citep{MirzasoleimanBK16, MirzasoleimanKS16, MirzasoleimanZK16} and differential privacy \citep{ChaturvediNZ21, MitrovicB0K17,RafieyY20}.
In many of the above applications $N$ (the number of underlying submodular functions) is extremely large, making the evaluation of $F$ prohibitively slow. We illustrate this with a simple example.
%We focus on the case where $\forall i, f^i$ is monotone ($\forall S \subseteq T\subseteq E, f^i(T) \geq f^i(S)$), which in turn means $F$ is monotone.

%\paragraph{Example 1 (Examplar-based learning)} You are given a large set of images, $X$, and your goal is to extract a small subset of images, $S \subset X$, that best represent the original set (e.g., creating thumbnails for a Youtube video). Intuitively, this is a submodular optimization problem, as the more images $S$ contains, the less we gain by adding an additional image. Furthermore, this function is decomposable as we can write $F(S) = \sum_{i=1}^N f^i(S)$, where $f^i(S)$ represents to what extent the set $S$ represent the $i$-th element of $X$. 
%One can also intuitively note that for most reasonable functions the above is monotone (e.g., more thumbnails achieve better ``coverage'' to the frames of the video).

\paragraph{Example (Welfare maximization)} Imagine you are tasked with deciding on a meal menu for a large group of $N$ people 
(e.g., all students in a university, all high school students in a country). 
You need to choose $k$ ingredients to use from a predetermined set 
(chicken, fish, beef, etc.) of size $n$. Every student has a specific preference, 
modeled as a monotone submodular function $f^i$. Our goal is to maximize the ``social welfare'', $F(S)=\sum f^i(S)$,  over all students. %Decomposable submodular play a pivotal role in welfare optimization during combinatorial auctions \citep{DobzinskiS06,Feige09, FeigeV06, PapadimitriouSS08, Vondrak08} 

\paragraph{The greedy algorithm}
Going forward we focus on the case where $\forall i, f^i$ is monotone ($\forall S \subseteq T\subseteq E, f^i(T) \geq f^i(S)$), which in turn means $F$ is monotone. This is applicable to the above example (e.g., students are happier with more food varieties).
For ease of presentation let us first focus on maximizing $F$ under a cardinality constraint $k$, i.e., $\max F(S), \size{S}\leq k$. The classical greedy algorithm \citep{NemhauserWF78} achieves an \emph{optimal} $(1-1/e)$-approximation for this problem. For $S,A\subseteq E$ we define $F_S(A) = F(S+ A) - F(S)$. We slightly abuse notation and write $F_S(e), F(e)$  instead of $F_S(\set{e}), F(\set{e})$. 
%The classical greedy algorithm starts

\paragraph{\greedy:} Start by initializing an empty set \( S_1 \). Then for each \( j \) from 1 to \( k \), perform the following steps: First, identify the element \( e' \) not already in \( S_j \) that maximizes the function \( F_{S_j}(e) \). Add this element \( e' \) to set \( S_j \) to form the new set \( S_{j+1} \). After completing all iterations, return the set \( S_{k+1} \).

\iffalse
\RestyleAlgo{boxruled}
\LinesNumbered
\begin{algorithm}[htbp]
	\DontPrintSemicolon
	\SetAlgoLined
$S_1 \gets \emptyset$\\
\For{$j=1$ to $k$}
{
$e' = \arg \max_{e\in E\setminus S_j} F_{S_j}(e)$\\
$S_{j+1} = S_j +e'$\\
}
return $S_{k+1}$\\
\caption{Greedy submodular maximization under a cardinality constraint}\label{alg: greedy}
\end{algorithm}
\fi
When $F$ is decomposable, and each evaluation of $f^i$ is counted as an oracle call, the above algorithm requires $O(Nnk)$ oracle calls. This can be prohibitively expensive if $N \gg n$. Looking back at our example, \greedy will require asking every student $k$ questions. That is, we would start by asking all students: ``Rate how much you would like to see food $x$ on the menu'' for $n$ possible options. We would then need to wait for all of their replies, and continue with ``Given that chicken is on the menu, rate how much you would like to see food $y$ on the menu'' for $n-1$ possible options and so on. Even if we do an online poll, this is still very time-consuming (as we must wait for \emph{everyone} to reply in each of the $k$ steps). If we could \emph{sample} a representative subset of the students, we could greatly speed up the above process. Specifically, we would like to eliminate the  dependence on $N$.% in \greedy. 

%\paragraph{Uniform sampling?} Perhaps the simplest approach would be to try and uniformly sample a sufficiently large subset of $f^i$. Unfortunately, we cannot get worst-case theoretical guarantees for uniform sampling. Consider the case where all $f^i$'s are always 0 and only a single $f^j$ takes non-zero values. Clearly, uniform sampling will almost surely miss $f^j$. 

%However, this is a pathological case that is very unlikely to occur in practice. To bridge this gap we go beyond worst-case and consider the smoothed complexity of this problem. The crux here is actually defining a realistic model of smoothing.

\paragraph{A sparsifier based approach}  Recently, \citet{rafiey2022sparsification} were the first to consider constructing a \emph{sparsifier} for $F$. That is, given a parameter $\eps >0$ they show how to find a vector $w\in \mathbb{R}^N$ such that the number of non-zero elements in $w$ is small in expectation and the function $\hat{F} = \sum_{i=1}^N w_i f^i$ satisfies with high probability (w.h.p)\footnote{Probability at least $1-1/n^c$ for an arbitrary constant $c>1$. The value of $c$ does affect the asymptotics of the results we state (including our own).} that $\forall S \subseteq E, (1-\eps)F(S) \leq \hat{F}(S) \leq (1+\eps) F(S)$. Where the main bottleneck in the above approach is computing the sampling probabilities for computing $w$. They treat this step as a \emph{preprocessing} step. The current state of the art construction is due to \citep{kraut} where a sparsifier of size $O(k^2n\eps^{-2})$ (where $k$ bounds the size of the solution) can be constructed using $O(Nn)$ oracle calls (see Appendix~\ref{sec: sparsifiers} for an in-depth overview of existing sparsifier constructions).
%We provide an in-depth overview of existing sparsifier constructions in Appendix~\ref{sec: sparsifiers}.

\paragraph{Weighted mini-batch} Mini-batch methods are at the heart of several fundamental machine learning algorithms (e.g., mini-batch k-means, SGD). Surprisingly, a mini-batch approach for this problem was not considered.
We present the first mini-batch algorithm for this problem, and show that it is superior to the sparsifier based approach both in theory and in practice. Roughly speaking, we show that sampling a new batch every iteration of the algorithm is ``more stable" compared to the sparsifier approach. This allows us to sample less elements in total and reduce the overall complexity. The main novelty in our analysis is using the greedy nature of the algorithm and carefully balancing an additive and a multiplicative error term.
Our sampling probabilities are the same as those of \cite{kraut}, and therefore, this algorithm pays the expensive $O(Nn)$ preprocessing time.

\paragraph{Beyond worst-case analysis}
When conducting our experiments, we added a simple baseline for both the sparsifier and the mini-batch algorithm -- instead of using weighted sampling we used \emph{uniform sampling} (previous results neglected this baseline \cite{rafiey2022sparsification}). Surprisingly, we observe that it outperforms weighted sampling both for the sparsifier and mini-batch algorithms. This is remarkable, as uniform sampling is extremely simple to implement, and requires no preprocessing which removes the dependence on $N$ altogether.

 Unfortunately, we cannot get worst-case theoretical guarantees for uniform sampling. Consider the case where only a single $f^j$ takes non-zero values (all other $f^i$'s are always 0). Clearly, uniform sampling will almost surely miss $f^j$. However, this is a pathological case that is very unlikely to occur in practice. To bridge this gap we go beyond worst-case analysis and consider the \emph{smoothed complexity} of this problem. 

 Smoothed analysis was introduced by \cite{SpielmanT04} in an attempt to explain the fast runtime of the Simplex algorithm in practice, despite its exponential worst-case runtime. In the smoothed analysis framework, an algorithm is provided with an adversarial input that is perturbed by some random noise. The crux of smoothed analysis is often defining a realistic model of noise.
 %The crux here is actually defining a realistic model of smoothing. \Gcomment{Ref spielman teng? say smoothed is adversarial + noise}
 
 Our main contribution is defining two very natural smoothing models. The reason we define two models is because the first allows us to provide theoretical guarantees both for our mini-batch algorithms and for all existing sparsifier algorithms (under uniform sampling), but only lines up empirically with some of the datasets we use. For the second model we are only able to show theoretical guarantees for our mini-batch algorithm, but it agrees empirically with all of our datasets.

%To define our models we assume 
To define our models of smoothing we assume w.log that $\forall i\in [N], e\in E, f_\myemptyset^i(e) \in [0,1]$ (we can always achieve this by normalization if $\set{f_\myemptyset^i(e)}$ are upper bounded). Let $\phi\in [0,1], d\in [N]$ be parameters and let us denote $A_e = \set{f_\myemptyset^i(e)}_{i\in [N]}$.

\paragraph{Model 1}
It holds that $N=\Omega(\frac{d}{\phi}\log (nd))$, and for \emph{every} $e \in E$ the following two conditions hold: (1) Every $f^i(e)\in A_e$ is a random variable such that $\mathbb{E}[f^i(e)] \geq \phi$. (2) Elements in $A_e$ have dependency at most $d$ (every $f_\myemptyset^i(e)$ depends on at most $d$ other elements in $A_e$). Note that we can have arbitrary dependencies between elements in $A_e, A_{e'}, e\neq e'$. 
%It holds that $N=\Omega(\frac{d}{\phi}\log (nd))$, and for \emph{every} $e \in E$ the following two conditions hold:(1) All elements in $A_e$ are chosen from some distribution $\mathcal{P}_e$ with mean at least $\phi$. (2) Elements in $A_e$ have dependency at most $d$ (every $f_\myemptyset^i(e)$ depends on at most $d$ other elements in $A_e$). Note that we can have arbitrary dependencies between elements in $A_e, A_{e'}, e\neq e'$. 
%(3) $N=\Omega(\frac{d}{\phi}\log n)$.
%(1) There exists a set $A_e\subseteq \set{f_\myemptyset^i(e)}_{i\in [N]}, \size{A_e}\geq \delta N$ such that $\forall i\in I_e, f_\myemptyset^i(e)$ are chosen from some distribution $\mathcal{P}$ with mean at least $\phi$. All other elements may be chosen adversarially.
%(2) with mean at least $\phi$.
%(2) Elements in $A_e$ have dependency at most $d$ (every $f_\myemptyset^i(e)$ depends on at most $d$ other elements in $A_e$). Note that we can have arbitrary dependencies between elements in $A_e, A_{e'}, e\neq e'$.

\paragraph{Model 2} Identical to Model 1, except that there \emph{exists} $e\in E$ such that conditions (1) and (2) hold.

\paragraph{Intuition} 
Going back to our lunch menu example, Assumption (1) of Model 1 means that every possible food on the menu is not universally hated by the students. Assumption (2) means that the preference of a student for any specific food is sufficiently independent of the preferences of other students. Assuming that $N=\Omega(\frac{d}{\phi}\log (nd))$ means that we have a sufficiently large student body -- note that increasing $N$ should not change $\phi$ and $d$.
Model 2 only requires assumptions (1) and (2) to hold for a single menu item. Intuitively, Model 1 assumes that all food choices are ``not too bad" while Model 2 assumes that there exists at least one such choice.

\paragraph{Comparison to other model of smoothing} 
Perhaps the most general smoothing approach when dealing with weighted inputs (e.g., in [0,1]) is to assume that the weights are taken from some distribution whose density is upper bounded by a smoothing parameter $\phi$ \cite{EtscheidR17, AngelBPW17}. This generalizes the approach of \citet{SpielmanT04}, where Guassian noise was added to the weights. Our approach is even more general, as the above immediately implies that the expectation is lower bounded by $\phi$. Furthermore, we only assume bounded independence and Model 2 only partially smoothes the input. %Furthermore, 

%We would like to emphasize that our approach to smoothing are actually more general than other popular approaches. For example, when dealing with weighted inputs (e.g., in [0,1]) it is common to assume that the weights are taken from some distribution whose density is upper bounded by a smoothing parameter $\phi$ \cite{EtscheidR17}. We note that this immediately implies that the expectation is lower bounded by $\phi$.

%\paragraph{Results overview} We show that under Model 1, uniform sampling achieves the same guarantees as weighted sampling for all existing sparsifier algorithms and our new mini-batch algorithms (with an added multiplicative $O(1 /\phi)$ factor in the query complexity). However, when empirically evaluating the predictions of the model we observe that it only successfully models some of the datasets. The much weaker Model 2 does successfully model all datasets, but we can only provide theoretical guarantees for the mini-batch algorithm under this model. 

While our primary contribution lies in providing theoretical guarantees for the uniform mini-batch algorithm (which empirically outperforms all other methods), we begin by presenting our results for weighted sampling. This will lay the groundwork to seamlessly prove our main results in Section~\ref{sec: smoothed}.
\subsection{The mini-batch algorithm}
\label{sec: results}
We focus on the \emph{greedy algorithm} for constrained submodular maximization.
We show that instead of sparsifying $F$, better results can be achieved by using mini-batches during the execution of the greedy algorithm. That is, rather than sampling a large sparsifier $\hat{F}$ and performing the optimization process, we show that if we sample a much smaller sparsifier (a \emph{mini-batch}), $\hat{F}^j$, for the $j$-th step of the greedy algorithm, we can achieve improved results both in theory (Table~\ref{table:results}) and in practice (Section~\ref{sec: experiments}). Most notably, we observe that the mini-batch approach is superior to the sparsifier approach for small batch sizes on various real world datasets. This is also the case when we combine our approach with the popular \emph{stochastic-greedy} algorithm \citep{BuchbinderFS15, MirzasoleimanBK15}.

While the mini-batch approach results in an improvement in performance, the sparsifier approach has the benefit of being independent of the algorithm. That is, while any approximation algorithm executed on a sparsifier immediately achieves (nearly) the same guarantees for the original function, we need to re-establish the approximation ratio of our mini-batch algorithm for different constraints. Although these proofs are often straightforward, compiling an exhaustive list of where the mini-batch method is applicable is both laborious and offers limited insights.
%While these proofs are often straightforward, it is quite tedious to present a complete list of problems solved by our approach. 

We focus on two widely researched constraints: the cardinality constraint and the \( p \)-system constraint. The cardinality constraint was chosen for its simplicity, while the \( p \)-system constraint was chosen for its broad applicability. %\Gcomment{define curvature} %We strongly believe that our approach could be applied beyond submodular functions (e.g., \( k \)-submodular functions, similar to \citet{standa2023}), achieve better approximation guarantees for specific constraints, and even applied beyond the greedy algorithm.
%While our results do not require bounded curvature, we achieve improved performance when the curvature is bounded.

%We note that while the results of \citet{standa2023} achieve the same oracle query complexity during preprocessing (only for the bounded curvature case), their preprocessing step is much more complicated than ours, and has runtime that depends on $\log \max_{i\in [N]}\frac{\max_{e\in E} f^i_\myemptyset(e)}{\min_{e\in E} f^i_\myemptyset(e)}$. Furthermore, they assume every $f^i$ has bounded curvature, while our results only require $F$ to have bounded curvature.

%While computing a sparsifier is more general in the sense that it is independent of the algorithm, we show that a significant improvement in performance can be achieved for the widely used greedy approach. 
\paragraph{Theoretical results} We compare our results with the state of the art sparsifier results and the naive algorithm (without sampling or sparsification) in Table~\ref{table:results}\footnote{Where $\widetilde{O}$ hides $\log n$ factors.}. We can get improved performance if the curvature of $F$ is bounded\footnote{The \emph{curvature} of a submodular function $F$ is defined as $c = 1-\min_{S \subseteq E, e\in E \setminus S} \frac{F_S(e)}{F_{\myemptyset}(e)}$. We say that $F$ has \emph{bounded-curvature} if $c < 1$.}. Note that the ``Uniform" column requires no preprocessing, and the query complexity differs by a multiplicative $\Theta(1/n\phi)$-factor.
%It's worth noting that while \citet{standa2023} assume every $f^i$ has bounded curvature, we only require $F$ to have bounded curvature.
In Section~\ref{sec: analysis} we prove our results for the ``Weighted" column and in 
Section~\ref{sec: smoothed} we prove our results for the ``Uniform" column. Under Model 2 our results for the unbounded curvature mini-batch case still hold for uniform sampling (all other results hold under Model 1). We empirically observe that $\phi =O(1)$ for our datasets (Section~\ref{sec: smoothed}), which explains the superior performance of the uniform mini-batch algorithm in practice. 

%we show that under Model 1 all of the results below also hold for \emph{uniform} sampling with an additional $O(1/\phi)$ multiplicative factor in the query complexity. Under Model 2 our results for the unbounded curvature mini-batch case still hold under uniform sampling. We emphasize that uniform sampling does not require any preprocessing. %\Gcomment{query complexity?}
%While \citet{rafiey2022sparsification,standa2023} require all $f^i$ to be monotone, we only require $F$ to be monotone.

%Our main contribution is the showing that the mini-batch approach is much more efficient than sparsification for maximizing a non-negative monotone submodular function under a set of constraints. 

\begin{table}[h]
\centering
\renewcommand{\arraystretch}{1.5} % Adjust the factor as needed
\begin{tabular}{|c|c|c|c|c|}
\hline
          & \multicolumn{2}{|c|}{Preprocessing} & \multicolumn{2}{|c|}{Oracle queries} \\
\cline{2-5}
          & Weighted & \textbf{Uniform} & Weighted & \makecell{\textbf{Uniform} \\ (Model 1)} \\
\hline
Naive & \multicolumn{2}{|c|}{None} & \multicolumn{2}{|c|}{$O(Nnk)$} \\
\hline
\makecell{\citeauthor{kraut}} & $O(Nn)$   & None   & $\widetilde{O}\left( \frac{k^3 n^2}{\eps^2} \right)$ & $\widetilde{O}\left( \frac{k^3 n}{\eps^2 \phi} \right)$ \\
\hline
\makecell{\textbf{Our results}
\\ Uniform holds for 
\\Model 1 \& 2} & $O(Nn)$   & None   & \makecell{Card. $\widetilde{O}\left( \frac{k^2n^2}{\eps^2} \right)$ \\ $p$-sys. $\widetilde{O}\left( \frac{k^2pn^2}{\eps^2} \right)$} & \makecell{Card. $\widetilde{O}\left( \frac{k^2n}{\eps^2 \phi} \right)$ \\ $p$-sys. $\widetilde{O}\left( \frac{k^2pn}{\eps^2 \phi} \right)$} \\
\hline
\makecell{\textbf{Our results} \\ (bounded curvature)} & $O(Nn)$   & None   & $\widetilde{O}\left( \frac{kn^2}{(1-c)\eps^2} \right)$ & $\widetilde{O}\left( \frac{kn}{(1-c)\eps^2 \phi} \right)$ \\
\hline
\end{tabular}

\caption{Comparison of the number of oracle queries during preprocessing and during execution. Results are for the greedy algorithm under both a cardinality constraint and a $p$-system constraint. Unless explicitly stated the number of queries is the same for both constraints. All results achieve the near optimal approximation guarantees of $(1-1/e-\epsilon)$ for a cardinality constraint and $(\frac{1-\eps}{p+1})$ for a $p$-system constraint.
%Note that $B$ might be \emph{exponential} in $n$ (see Appendix~\ref{sec: sparsifiers} for an exact definition).
}
\label{table:results}
\end{table}

\paragraph{Meta greedy algorithm} Our starting point is the meta greedy algorithm (Algorithm~\ref{alg: meta minibatch greedy}). The algorithm executes for $k \leq n$ iterations where $k$ is some upper bound on the size of the solution. At every iteration, the set $A_j\subseteq E\setminus S_j$ represents some constraint that limits the choice of potential elements to extend $S_j$. The algorithm terminates either when the solution size reaches 
$k$ or when no further extensions to the current solution are possible (i.e., $A_j = \emptyset$).
Furthermore, the algorithm does not have access to the exact \emph{incremental oracle}, $F_{S_j}$, at every iteration, but only to some approximation (which may differ between iterations).
\RestyleAlgo{boxruled}
\LinesNumbered
\begin{algorithm*}[htbp]
	\DontPrintSemicolon
	\SetAlgoLined
$S_1 \gets \emptyset$ \\
Let $k$ be an upper bound on the size of the solution\\
\For{$j=1$ to $k$}
{
Let $A_j\subseteq E\setminus S_j$\hspace{6pt}  $\triangleright$ \text{Problem specific constraint (e.g., $A_j = E\setminus S_j$ for card. constraint)}\\ 
\lIf{$A_j = \emptyset$}{return $S_j$}
Let $\hat{F}^j_{S_j}$ be an approximation for $F_{S_j}$ \label{line: meta get approx oracle} \hspace{5pt} $\triangleright$ \text{Problem specific approximation}\\ 
$e_j = \arg \max_{e\in A_j} \hat{F}^j_{S_j}(e)$\\
$S_{j+1} = S_j +e_j$\\
}
return $S_{k+1}$
\caption{Meta greedy algorithm with an approximate oracle}\label{alg: meta minibatch greedy}
\end{algorithm*}
Before we formally define ``approximation'', let us note that when we have access to exact values of $F_{S_j}$, Algorithm~\ref{alg: meta minibatch greedy} captures many variants of the greedy submodular maximization algorithm. 
For example, setting $A_j = E\setminus S_j$ we get \greedy. This meta-algorithm also captures the case of maximization under a \emph{$p$-system constraint}. For ease of presentation we defer the discussion about $p$-systems to Appendix~\ref{sec: psys}.
%and focus on cardinality constraints going forward.

%For example, setting $A_j = E\setminus S_j$ we get the algorithm of \citet{NemhauserWF78} for maximizing a non-negative submodular function under a cardinality constraint. This meta-algorithm also captures the case of maximization under a \emph{$p$-system constraint}.

\paragraph{Approximate oracles} In many scenarios we do not have access to \emph{exact} values of $F_{S_j}$, and instead we must make do with an approximation. We start with the notion of an \emph{approximate incremental oracle} introduced in \citep{goundan2007revisiting}. We say that $\hat{F}^j_{S_j}$ is an $(1-\eps)$-approximate incremental oracle if $\forall e\in A_j, (1-\epsilon) F_{S_j}(e) \leq \hat{F}^j_{S_j}(e) \leq (1+\epsilon) F_{S_j}(e)$.
\iffalse
\[
\forall e\in A_j, (1-\epsilon) F_{S_j}(e) \leq \hat{F}^j_{S_j}(e) \leq (1+\epsilon) F_{S_j}(e)
\]
\fi
It was shown in \citep{goundan2007revisiting, CalinescuCPV11}\footnote{Strictly speaking, both \citet{goundan2007revisiting} and \citet{CalinescuCPV11} define the approximate incremental oracle to be a function that returns $e_j$ at iteration $j$ of the greedy algorithm such that
$\forall e\in A_j, F_{S_j}(e_j)\geq (1-\eps)F_{S_j}(e)$. Our definition guarantees this property while allowing easy analysis of the mini-batch algorithm.} that given a $(1-\eps)$-approximate incremental oracle, the greedy algorithm under both a cardinality constraint and a $p$-system constraint achieves almost the same (optimal) approximation ratio as the non-approximate case. 
%\citet{CalinescuCPV11} show that the greedy algorithm for maximizing a non-negative submodular function under a $p$-system constraint achieves almost the same (optimal) approximation ratio as the non-approximate case.
\begin{restatable}{theorem}{metaApproxKnown}
\label{thm:meta approx known guarantees}
    Algorithm~\ref{alg: meta minibatch greedy} with an $(1-\epsilon)$-approximate incremental oracle has the following guarantees w.h.p: (1) It achieves a $(1-1/e-\eps)$-approximation under a cardinality constraint $k$ \citep{goundan2007revisiting}. (2) It achieves a $(\frac{1-\eps}{1+p})$-approximation under a $p$-system constraint \citep{CalinescuCPV11}.
    \iffalse
    \begin{itemize}
        \item It achieves a $(1-1/e-\eps)$-approximation under a cardinality constraint $k$ \citep{goundan2007revisiting}.
        \item It achieves a $(\frac{1-\eps}{1+p})$-approximation under a $p$-system constraint \citep{CalinescuCPV11}.
    \end{itemize}
    \fi
\end{restatable}
We introduce a weaker type of approximate incremental oracle, which we call an \emph{additive} approximate incremental oracle. We extend the results of Theorem~\ref{thm:meta approx known guarantees} for this case. Let $S^*$ be some optimal solution for $F$ (under the relevant set of constraints). We say that $\hat{F}^j_{S_j}$ is an \emph{additive} $\eps'$-approximate incremental oracle if $\forall e\in A_j, F_{S_j}(e)-\epsilon' F(S^*) \leq \hat{F}^j_{S_j}(e) \leq  F_{S_j}(e)+\epsilon' F(S^*)$.
\iffalse
\[
\forall e\in A_j, F_{S_j}(e)-\epsilon' F(S^*) \leq \hat{F}^j_{S_j}(e) \leq  F_{S_j}(e)+\epsilon' F(S^*)
\]
\fi

This might seem problematic at first glance, as it might be the case that $F(S^*) \gg F_{S_j}(e)$. Luckily, the proofs guaranteeing the approximation ratio are \emph{linear} in nature. Therefore, by the end of the proof we end up with an expression of the form: $F(S_{k+1}) \geq F(S^*)\beta - \gamma \eps' F(S^*)$.
\iffalse
\[
F(S_{k+1}) \geq F(S^*)\beta - \gamma \eps' F(S^*)
\]
\fi
 Where $\beta$ is the desired approximation ratio and $\gamma$ depends on the parameters of the problem (e.g., $\beta = (1-1/e), \gamma=2k$ for a cardinality constraint). We can achieve the desired result by setting $\eps' = \eps/\gamma$. We state the following theorem (the proofs are similar to those of \citep{goundan2007revisiting, CalinescuCPV11}, and we defer them to the Appendix).
\begin{restatable}{theorem}{metaApprox}
\label{thm:meta approx guarantees}
    Algorithm~\ref{alg: meta minibatch greedy} with an additive $\eps'$-approximate incremental oracle has the following guarantees w.h.p: (1) If $\eps' < \eps/2k$, it achieves a $(1-1/e-\eps)$-approximation under a cardinality constraint $k$. (2) If $\eps' < \eps/2kp$, it achieves a $(\frac{1-\eps}{1+p})$-approximation under a $p$-system constraint.
    \iffalse
    \begin{itemize}
        \item If $\eps' < \eps/2k$, it achieves a $(1-1/e-\eps)$-approximation under a cardinality constraint $k$.
        \item If $\eps' < \eps/2kp$, it achieves a $(\frac{1-\eps}{1+p})$-approximation under a $p$-system constraint.
    \end{itemize}
    \fi
\end{restatable}

\paragraph{Mini-batch sampling}\label{par: minibatch} Although we use the same sampling probabilities as the sparsifier approach, instead of sampling a single $\hat{F}$ at the beginning, we sample a new $\hat{F}^j$ (\emph{mini-batch}), for every step of the algorithm. Recall that $\hat{F}^j_{S_j}(e) = \hat{F}^j(S_j + e) - \hat{F}^j(S_j )$.  

We show that when $\hat{F}^j_{S_j}$ is sampled using mini-batch sampling we indeed get, w.h.p, an (additive) approximate incremental oracle for every step of the algorithm. We present our sampling procedure in Algorithm~\ref{alg: approx inc oracale} and the complete mini-batch algorithm in Algorithm~\ref{alg: minibatch greedy}. 
%It takes in a batch size parameter $\alpha$ and samples every $f^i$ with probability proportional to $\alpha p_i$. 
%The main benefit in our approach is that its is sufficient to set $p_i = \max_{e \in E, F_\myemptyset(e)\neq 0} \frac{f^i_\myemptyset(e)}{F_\myemptyset(e)}$ compared to $\max_{S\subseteq E, F(S)\neq 0} \frac{f^i(S)}{F(S)}$ in \citep{rafiey2022sparsification}. This only requires $O(Nn)$ oracle evaluations.

%Similar to \citep{rafiey2022sparsification, kraut} we treat the computation of the $p_i$'s as a preprocessing step. The justification for this, is that the $p_i$'s do not depend on the \emph{constraints} of the problem. Therefore, computing the $p_i$'s a single time, we can execute our algorithm on various constraints (e.g., different $p$-systems). 
%In Section~\ref{sec: experiments} we experimentally observe that even when we take the preprocessing time into account our running tim
%Using the mini-batch approach has several advantages compared to computing a sparsifier. Specifically, our preprocessing step only requires computing $\max_{e \in E, F_\myemptyset(e)\neq 0} \frac{f^i_\myemptyset(e)}{F_\myemptyset(e)}$, which only requires $O(Nn)$ oracle evaluations\footnote{We ignore the degenerate case where $\forall e, F_\myemptyset(e)=0$.}. %Second, the expected size of every mini-batch has no dependence on $B$. 

\begin{minipage}[t]{0.4\textwidth}
\begin{algorithm}[H]
    \RestyleAlgo{boxruled}
    %\LinesNumbered
    \DontPrintSemicolon
    \SetAlgoLined
    $w \gets \vec{0}$\;
    \For{$i=1$ to $N$}{
        $\alpha_i \gets\min\{1, \alpha p_i\}$\;
        $w_i \gets1/\alpha_i$ with probability $\alpha_i$ %\hspace{10pt} $\triangleright$ \text{Do nothing with probability $1-\alpha_i$}
    }
    return $\hat{F} = \sum_{i=1}^N w_i f^i$
    \caption{Sample($\alpha, \set{p_i}_{i=1}^N$)}\label{alg: approx inc oracale}
\end{algorithm}
\end{minipage}%
\hspace{30pt}
\begin{minipage}[t]{0.5\textwidth}
\begin{algorithm}[H]
    \RestyleAlgo{boxruled}
    %\LinesNumbered
    \DontPrintSemicolon
    \SetAlgoLined
    $\forall i\in [N], p_i \gets \max_{e \in E, F_\myemptyset(e)\neq 0} \frac{f^i_\myemptyset(e)}{F_\myemptyset(e)}$ \\
    \tcp{Uniform sampling: $p_i = 1/N$}
    %\hspace{10pt} $\triangleright$ \\\text{Preprocessing}\\
    $\alpha$ is a batch parameter\\
    $S_1 \gets \emptyset$ \\
    $k$ is an upper bound on the size of the solution\\
    \For{$j=1$ to $k$}{
        Let $A_j\subseteq E\setminus S_j$ \\%\hspace{20pt} $\triangleright$ \text{Problem specific constraint}\\ 
        \lIf{$A_j = \emptyset$}{return $S_j$}
        Let $\hat{F}^j\gets Sample(\alpha, \set{p_i}_{i=1}^N)$\\ 
        $e_j = \arg \max_{e\in A_j} \hat{F}^j_{S_j}(e)$\\
        $S_{j+1} = S_j + e_j$\\
    }
    return $S_{k+1}$
    \caption{Mini-batch greedy}\label{alg: minibatch greedy}
\end{algorithm}
\end{minipage}

%Plugging Algorithm~\ref{alg: approx inc oracale} into Line~\ref{line: meta get approx oracle} of Algorithm~\ref{alg: meta minibatch greedy} we get our \emph{mini-batch greedy algorithm}. That is, in the $j$-th iteration, we call Algorithm~\ref{alg: approx inc oracale}, get back $\hat{F}$ and set $\hat{F}^j_{S_j}(e)=\hat{F}_{S_j}(e)$. 

In Section~\ref{sec: analysis} we analyze the relation between the batch parameter, $\alpha$, and the the type of approximate incremental oracles guaranteed by our sampling procedure. We state the main theorem for the section below.
\begin{restatable}{theorem}{minibatchApprox}
\label{thm:minibatch approx guarantees}
    The mini-batch greedy algorithm (Algorithm~\ref{alg: minibatch greedy}) maximizing a non-negative monotone submodular function has the following guarantees:
    \begin{enumerate}
        \item If $F$ has curvature bounded by $c$, and $\alpha = \Theta(\frac{\log n}{\eps^2 (1-c)})$ it holds w.h.p that $\forall j\in [k]$ that $\hat{F}^j_{S_j}$ is an $(1-\epsilon)$-approximate incremental oracle.
        \item If $\alpha = \Theta(\eps^{-2} \gamma\log n)$ it holds w.h.p that $\forall j\in [k]$ that $\hat{F}^j_{S_j}$ is an additive $(\epsilon/\gamma)$-approximate incremental oracle, for any parameter $\gamma>0$.
    \end{enumerate}
    Furthermore, the number of oracle evaluations during preprocessing is $O(nN)$ and an expected $\alpha (\sum_{i=1}^N p_i) (\sum_{j=1}^k \size{A_j})  = O(\alpha k n^2)$ during execution.
\end{restatable}

%Note that we assume $c$ is known in the above Theorem. In the end of Section~\ref{sec: analysis} we show that as long as $1/c = poly(n)$ a simple doubling argument allows us to avoid this assumption without affecting our results asymptotically.

Combining Theorem~\ref{thm:minibatch approx guarantees} (setting $\gamma = k$ for a cardinality constraint and $\gamma = kp$ for a $p$-system constraint) with Theorem~\ref{thm:meta approx known guarantees} and Theorem~\ref{thm:meta approx guarantees} we state our main result.
%and noting that we have $k$ iterations, each requiring at most $n$ calls to $\hat{F}^j_{S_j}$ we get our main result:

\begin{restatable}{theorem}{finalResults}
\label{thm:final res}
    The mini-batch greedy algorithm maximizing a non-negative monotone submodular function requires $O(nN)$ oracle calls during preprocessing and has the following guarantees:
    \begin{enumerate}
        \item If $F$ has curvature bounded by $c$, it achieves w.h.p a $(1-1/e -\eps)$-approximation under a cardinality constraint and $(\frac{1-\eps}{1+p})$-approximation under a $p$-system constraint with an expected $O(\frac{k n^2\log n}{\eps^2 (1-c)})$ oracle evaluations for both cases.
        \item It achieves w.h.p a $(1-1/e -\eps)$-approximation under a cardinality constraint and $(\frac{1-\eps}{1+p})$-approximation under a $p$-system constraint with an expected $O(k^2 (n /\eps)^2\log n)$ and $O(k^2p (n /\eps)^2\log n)$ oracle evaluations respectively.
    \end{enumerate}
\end{restatable}

\subsection{Related work}
%\paragraph{Bounded curvature}  By leveraging this property, algorithms can achieve better approximation guarantees, making them more effective in applications like machine learning, network design, and sensor placement.

\paragraph{Approximate oracles} Apart from the results of \citep{goundan2007revisiting,CalinescuCPV11} there are works that use different notions of an approximate oracle. Several works consider an approximate oracle $\hat{F}$, such that $\forall S\subseteq E, \size{\hat{F}(S) - F(S)} < \eps F(S)$ \citep{CrawfordKT19,HorelS16,QianS0TZ17}. The main difference of these models to our work is the fact that they do not assume the surrogate function, $\hat{F}$, to be submodular. This adds a significant complication to the analysis and degrades the performance guarantees. 

\paragraph{Mini-batch methods} The closest result resembling our mini-batch approach is the \emph{stochastic-greedy algorithm} \citep{BuchbinderFS15, MirzasoleimanBK15}. They improve the expected query complexity of the greedy algorithm under a cardinality constraint by only considering a small random subset of $E\setminus S_j$ at the $j$-th iteration. 
We note that their approach can be combined into our mini-batch algorithm, reducing our query complexity by a $\Tilde{\Theta}(k)$ factor, resulting in an approximation guarantee in expectation instead of w.h.p. 

\paragraph{Smoothed analysis} 
%Smoothed analysis was introduced by \cite{SpielmanT04} in an attempt to explain the fast runtime of the Simplex algorithm in practice, despite its exponential worst-case runtime. In the smoothed analysis framework, an algorithm is provided with an adversarial input that is perturbed by some random noise. 
To the best of our knowledge,  \cite{Rubinstein022} is the only result that considers smoothed analysis in the context submodular maximization. They consider submodular maximization under a cardinality constraint, where the cardinality parameter $k$ undergoes a perturbation according to some known distribution. 

\section{Analysis of the mini-batch greedy algorithm}
\label{sec: analysis}
%In this section we assume that $F$ has curvature bounded by $c<1$. 
%We present Algorithm~\ref{alg: approx inc oracale} which given $F, \set{p_i}_{i=1}^N,\eps, c$ computes an approximate incremental oracle. Recall that $p_i =\max_{e\in  E} f^i(e)/F(e)$.
%This can in turn be plugged into Line~\ref{line: get approx oracle} Algorithm~\ref{alg: minibatch greedy} and Line~\ref{line: psys get approx oracle} Algorithm~\ref{alg: minibatch greedy psys}.

%Let us start by bounding the expected size of $\hat{F}$ in Algorithm~\ref{alg: approx inc oracale}. 
%We start with the following useful lemma proved in \citet{kraut}. We restate the proof for completeness in the appendix. 
We start by with the following lemma from \citep{kraut}, which bounds the expected size of $\hat{F}$. We present a proof in the appendix for completeness. 
%\Gcomment{bring back the $\emptyset$ subscript everywhere?}
\begin{lemma}
    \label{lem: bound pi}
    The expected size of $\hat{F}$ is
    $\alpha \sum_{i=1}^N{p_i} \leq \alpha n$.
    %It holds that $\sum_{i=1}^N p_i \leq n$.
\end{lemma}
\iffalse
\begin{proof}
    We start by showing that $\sum_{i=1}^N p_i \leq n$.
    Let us divide the range $[N]$ into 
    $$A_e = \set{i\in N \mid e = \argmax_{e'\in E, F_\myemptyset(e')\neq 0} \frac{f^i_\myemptyset(e')}{F_\myemptyset(e')}}$$ If 2 elements in $E$ achieve the maximum value for some $i$, we assign it to a single $A_e$ arbitrarily. 
\begin{align*}
\sum_{i=1}^N p_i = \sum_{i=1}^N \max_{e\in E} \frac{f^i_\myemptyset(e)}{F_\myemptyset(e)} = 
\sum_{e\in E} \sum_{i\in A_e} \frac{f^i_\myemptyset(e)}{F_\myemptyset(e)} =  \sum_{e\in E} \frac{\sum_{i\in A_e} f^i_\myemptyset(e)}{F_\myemptyset(e)} \leq \sum_{e\in E} 1 \leq n    
\end{align*}
Let $X_i$ be an indicator variable for the event $w_i > 0$.
    We are interested in $\sum_{i=1}^N \mathbb{E}[X_i]$. It holds that:
    \[
     \sum_{i=1}^N \mathbb{E}[X_i] = \sum_{i=1}^N \alpha_i \leq \sum_{i=1}^N \alpha p_i = \alpha \sum_{i=1}^N p_i \leq \alpha n 
     %O(\frac{\log n}{(1-c)\eps^2} \sum_{i=1}^N{p_i}) = O(\frac{n\log n}{(1-c)\eps^2})
    \]
\end{proof}
\fi
Note that the above yields a tighter bound of $\alpha$ for uniform sampling ($p_i=1/N$).
Next, let us show that $\hat{F}$ returned by Algorithm~\ref{alg: approx inc oracale} is indeed an (additive) approximate incremental oracle w.h.p. We make use of the following Chernoff bound. %\Gcomment{ref and remove proof}
\begin{theorem}[Chernoff bound \citep{MotwaniR95}]\label{thm:chernoff}

Let $X_1,...,X_N$ be independent random variables in the range $[0, a]$. Let $X = \sum_{i=1}^{N}X_i$. Then for any $\eps \in [0, 1]$ and $\mu \geq \mathbb{E}[T]$ it holds that $\mathbb{P}(|X - \mathbb{E}[X]| \geq \eps \mu) \leq 2 \exp \left(-\eps^2 \mu/3a \right)$.
\iffalse
\begin{equation*}
\label{Chernoff_bound}
\mathbb{P}(|X - \mathbb{E}[X]| \geq \eps \mu) \leq 2 \exp \left(-\frac{\eps^2 \mu}{3a} \right)
\end{equation*}
\fi
\end{theorem}
\iffalse
\begin{proof}
    Denote $\mu = \beta \mathbb{E}[x], \beta >1$. It holds that \Gcomment{ref Chernoff bound}:
    \begin{equation*}
\mathbb{P}(|X - \mathbb{E}[X]| \geq \eps \mu/\beta) = \mathbb{P}(|X - \mathbb{E}[X]| \geq \eps \mathbb{E}[x]) \leq 2 \exp \left(-\frac{\eps^2 \mathbb{E}[X]}{3a} \right) = 2 \exp \left(-\frac{\eps^2 \mu}{3a\beta} \right)
\end{equation*}
Setting $\eps = \eps'\beta$ we get:
\begin{equation*}
\mathbb{P}(|X - \mathbb{E}[X]| \geq \eps' \mu) \leq 2 \exp \left(-\frac{\eps'^2 \mu \beta}{3a} \right) \leq 2 \exp \left(-\frac{\eps'^2 \mu}{3a} \right)
\end{equation*}
\end{proof}
\fi

%Let us start with the following auxiliary lemma.
The following lemma provides concentration guarantees for $\hat{F}$ in Algorithm~\ref{alg: approx inc oracale}. 
\begin{lemma}
\label{lem: concentration lemma}
For every $S\subseteq E$ ($\hat{F}$ sampled after $S$ is fixed) and for every $e\in E$ and $\mu \geq F_S(e)$, it holds that $\mathbb{P}\left[|\hat{F}_S(e) - F_S(e)| \geq \eps \mu \right]
\leq 2 \exp \left(-\frac{\eps^2 \mu}{3 F_\myemptyset(e) / \alpha} \right)$.
\iffalse
\begin{align*}
%\label{Chernoff_bound_epsilon_sparsifier}
 \mathbb{P}\left[|\hat{F}_S(e) - F_S(e)| \geq \eps \mu \right]
\leq 2 \exp \left(-\frac{\eps^2 \mu}{3 F_\myemptyset(e) / \alpha} \right)
%= 2 \exp \left(-\frac{\alpha \eps^2 (1-c)}{3}\right) 
%\label{chernoff_bound_epsilon_sparsifier2}
\end{align*}
\fi
\end{lemma}
\begin{proof}
Fix some $e\in E$.
Let $G = \sum_{i\in I} f^i$, where $I = \{i\in [N] \mid \alpha_i = 1\}$. Let $F'_S(e) = F_S(e) - G_S(e)$ and $\hat{F}'_S(e)=\hat{F}_S(e) - G_S(e)$. Let $J = [N] \setminus I$. 
It holds that:
\begin{align*}
\mathbb{P}\left[|\hat{F}_S(e) - F_S(e)| \geq \eps \mu \right] 
= \mathbb{P}\left[|\hat{F}'_S(e) + G_S(e) - F'_S(e) - G_S(e)| 
 \geq \eps \mu \right] 
  = \mathbb{P}\left[|\hat{F}'_S(e) - F'_S(e)| \geq \eps \mu \right]
%& \leq 2 \exp \left(-\eps^2 \mu / 3a \right) 
%\label{chernoff_bound_k_2}
\end{align*}

Due to the fact that $\Expc{ w_i} = 1$ we have $\mathbb{E}[\hat{F}'_S(e)] = \mathbb{E}[\sum_{i\in J} w_i f^i_S(e)] 
    = F'_S(e)$. As $f^i$'s are monotone, it holds that $\mu \geq F_S(e) \geq F'_S(e)$. Applying a Chernoff bound (Theorem \ref{thm:chernoff}) we have
\begin{align*}
%\label{Chernoff_bound_k}
\mathbb{P}\left[|\hat{F}'_S(e) - F'_S(e)| \geq \eps \mu \right] \leq 2 \exp \left(-\eps^2 \mu / 3a \right) 
%\label{chernoff_bound_k_2}
\end{align*}
where $a = \max \{w_if^i_S(e)\}_{i\in J}$. Recall that $w_i = 1 / \alpha_i$ where $\alpha_i = \min\{1, \alpha  p_i$\} and $\alpha_i < 1$ for all $i \in J$. Let us upper bound $a$.
\begin{align*}
%\label{upper_bound_for_a}
a = \max_{i\in J}w_if^i_S(e) = \max_{i\in J}\frac{f^i_S(e)}{\alpha p_i}  = \max_{i\in J}\frac{f^i_S(e)}{\alpha \cdot \underset{e'\in E}{\max} \frac{f^i_\myemptyset( e')}{F_\myemptyset(e')}}  
\leq \max_{ i\in J}\frac{f^i_\myemptyset(e)}{\alpha \cdot \frac{f^i_\myemptyset(e)}{F_\myemptyset(e)}} = \frac{F_\myemptyset(e)}{\alpha} 
%\label{upper_bound_a_2}
\end{align*}
Where the inequality is due to submodularity and non-negativity in the numerator and maximality in the denominator. Note that the above is also correct if we only have some approximation to $p_i$ -- i.e., given $p'_i > p_i \lambda$ for $\lambda \in (0,1)$, we can increase $\alpha$ by a $1/\lambda$ factor and the above still holds.
%Where the first inequality is due to the fact that $p_i \geq \underset{e\in E}{\max} \frac{f^i(S +e)}{F(S_j +e)}$ (Lemma~\ref{lem: bound jpi}), and the second inequality uses the fact that $S\in N(S_j)$, so it can be written as $S_j +e$ for some $e\in E$.
%For the special case where $S = S_j$ we note that $jp_i \geq (j-1)p_i \geq \underset{e\in E}{\max} \frac{f^i(S_{j-1} +e)}{F(S_{j-1} +e)} \geq \frac{f^i(S_{j})}{F(S_{j})}$.
Given the above upper bound for $a$ we get:
\begin{align*}
%\label{Chernoff_bound_epsilon_sparsifier}
 \mathbb{P}\left[|\hat{F}_S(e) - F_S(e)| \geq \eps \mu \right]
\leq 2  \left(-\frac{\eps^2 \mu}{ 3a} \right) 
 \leq 2 \exp \left(-\frac{\eps^2 \alpha \mu}{3 F_\myemptyset(e) } \right) &\qedhere
%\leq 2 \exp \left(-\frac{\eps^2 F_S(e)}{3 F_\myemptyset(e) } \right) \leq 2 \exp \left(-\frac{\alpha \eps^2 (1-c)}{3} \right) 
%= 2 \exp \left(-\frac{\alpha \eps^2 (1-c)}{3}\right) 
%\label{chernoff_bound_epsilon_sparsifier2}
\end{align*}
%Recall that $\alpha = \Theta(\frac{ \log(n)}{\eps^{2}(1-c)})$. Hence, taking a union bound over all $n$ values of $e$ yields the desired result. \Gcomment{make consistent with the next lemma, e.g, write $< 1/n^3$.} 
%that Algorithm \ref{alg:sparsifier} with probability at least $1 - \delta$ returns a spectral sparsifier for $F$.
\end{proof}

Using the above we are ready to prove 
Theorem~\ref{thm:minibatch approx guarantees}.
%state the main theorem for this section.
%\minibatchApprox*

%\paragraph{Proof of Theorem~\ref{thm:minibatch approx guarantees}}
\begin{proof}[\textbf{Proof of Theorem~\ref{thm:minibatch approx guarantees}}]
The number of oracle evaluations is due to Lemma~\ref{lem: bound pi} and the fact that the algorithm executes for $k$ iteration and must evaluate $\size{A_j}\leq n$ elements per iteration.

Let us prove the approximation guarantees. Let us start with the bounded curvature case.
    Fix some $S_j$. As $\hat{F^j}$ is sampled after $S_j$ is fixed, we can fix some $e\in E$ and apply Lemma~\ref{lem: concentration lemma} with $\mu = F_{S_j}(e)$. We get that:
    \begin{align*}
     \mathbb{P}\left[|\hat{F}^j_{S_j}(e) - F_{S_j}(e)| \geq \eps F_{S_j}(e) \right] \leq 2 \exp \left(-\frac{\eps^2 \alpha F_{S_j}(e)}{3 F_\myemptyset(e) } \right) 
    \leq 2\exp \left(-\frac{\eps^2 \alpha (1-c)}{3} \right) \leq 1/n^3    
    \end{align*}
    
Where the second inequality is due to the fact that $F_{S_j}(e)/F_\myemptyset(e) \geq \min_{S \subseteq E, e'\in E \setminus S} F_{S}(e')/F_\myemptyset(e') = 1-c$, and the last transition is by setting an appropriate constant in $\alpha = \Theta(\frac{ \log(n)}{\eps^{2}(1-c)})$.
When the curvature is not bounded, we break the analysis into cases.

\paragraph{$F_\myemptyset(e) \leq \gamma F_{S_j}(e)$:} Setting $\mu = F_{S_j}(e)$ we get:
\begin{align*}
    \mathbb{P}\left[|\hat{F}^j_{S_j}(e) - F_{S_j}(e)| \geq \eps F_{S_j}(e) \right] \leq 2 \exp \left(-\frac{\eps^2 \alpha F_{S_j}(e)}{3 F_\myemptyset(e) } \right) 
    \leq 2\exp \left(-\frac{\eps^2 \alpha}{3\gamma} \right) \leq 1/n^3
\end{align*}
\paragraph{$F_\myemptyset(e) > \gamma F_{S_j}(e)$:}
Here we can set $\mu = F_\myemptyset(e)/\gamma \geq F_{S_j}(e)$ and get:
    \begin{align*}
     \mathbb{P}\left[|\hat{F}^j_{S_j}(e) - F_{S_j}(e)| \geq \eps F_\myemptyset(e) /\gamma \right] \leq 2 \exp \left(-\frac{\eps^2 \alpha F_\myemptyset(e)}{3\gamma F_\myemptyset(e) } \right) 
     \leq 2\exp \left(-\frac{\eps^2 \alpha }{3\gamma} \right) \leq 1/n^3  
    \end{align*}

Where in both cases the last inequality is by setting an appropriate constant in $\alpha = \Theta(\epsilon^{-2}\gamma \log n)$. Note that in the first case we get an $(1-\epsilon)$-approximate incremental oracle and in the second an additive $(\eps/\gamma)$-approximate incremental oracle. The second case is the worse of the two for our analysis.
For both bounded and unbounded curvature, we take a union bound over all $e\in E$ and $j\in[k]$ (at most $n^2$ values), which concludes the proof.
    %This is due to the value of $\alpha$ together with a union bound over all $j, $
\end{proof}
Note that in the above we use the fact that $F_\myemptyset(e) \leq F(S^*)$ (recall that $S^*$ is the optimal solution for $F$) to get the second result. This is sufficient for our proofs to go through, however, the theorem has a much stronger guarantee which might be useful in other contexts (as we will see in Section~\ref{sec: smoothed}).

%We note that the above assumes that we know $c$ to correctly set $\alpha$. However, note that even if we only know some upper bound $c'$ such that $c \leq c' < 1$, we can simply execute the algorithm $\tau =\ceil{\log 1/(1-c')}$ times, setting $c$ at iteration $t$ to $1-1/2^t$. We take the solution with the maximum value (evaluated according to $\hat{F}^k$ at the end of the $t$-th execution). Finally, we must increase $\alpha$ to guarantee that all sampled $\hat{F}$ maintain their approximation guarantee. Assuming that $1/(1-c') = O(poly(n))$, this does not affect our results asymptotically. \Gcomment{you have to execute until the last iteration...}

\section{Experiments}
\label{sec: experiments}
We perform experiments on the following datasets, where the goal for all datasets is to maximize $F(S)$ subject to $\size{S} \leq k$.

\paragraph{Uber pickups} 
This dataset consists of Uber pickups in New York city in May 2014\footnote{ \url{https://www.kaggle.com/fivethirtyeight/uber-pickups-in-new-york-city}}. The set contains $652,434$ records, where each record consists of a longitude and latitude, representing a pickup location. Following \citep{rafiey2022sparsification} we aim to find $k$ positions for idle drivers to wait from a subset of popular pickup locations. We formalize the problem as follows. 
We run Lloyd's algorithm on the dataset, $X$, and find 100 cluster centers. We set these cluster centers to be the ground set $E$. We define the goal function $F:2^E \rightarrow \mathbb{R}^+$ as $F(S) = \sum_{v\in X} f_v(S)$ where $f_v(S)=\max_{e\in E} d(v,e) - \min_{e\in S} d(v,e)$, and $d(v,e)$ is the Manhattan distance between $v$ and $e$. 

\paragraph{Discogs \citep{Kunegis13}} This dataset\footnote{\url{http://konect.cc/networks/discogs_lstyle/}} provides audio record information structured as a bipartite graph \( G = (L, R; E') \). The left nodes represent labels and the right nodes represent styles.
Each edge \( (u, v) \in L \times R \) signifies the involvement of a label \( u \) in producing a release of a style \( v \). The dataset comprises \( |R| = 383 \) labels, \( |L| = 243,764 \) styles, and \( |E'| = 5,255,950 \) edges. We aim to select \( k \) styles that cover the activity of the maximum number of labels. We construct a maximum coverage function \( F : 2^R \rightarrow \mathbb{R} \), where \( F(S) = \sum_{v \in L} f_v(S) \), and \( f_v(S) \) equals 1 if \( v \) is adjacent to some element of \( S \) and 0 otherwise. 
%The objective is to maximize \( F(S) \) subject to \( |S| \leq k \).

\paragraph{Examplar-based clustering} We consider the problem of selecting representative subset of $k$ images from a massive data set. We present experiments for both the CIFAR100 and the FashionMNIST datasets. For each dataset we consider a subset of $50,000$ images, denoted by $X$. We flatten every image into a one dimensional vector, subtract from it the mean of all images and normalize it to unit norm. We take the distance between two elements in $X$ as $d(x,x') = \norm{x-x'}^2$. Here the ground set is simply the dataset, $E=X$.
Similarly to the Uber pickup dataset we define the goal function $F:2^E \rightarrow \mathbb{R}^+$ as $F(S) = \sum_{v\in X} f_v(S)$ where $f_v(S)=\max_{e\in E} d(v,e) - \min_{e\in S} d(v,e)$.
%The goal of the problem is to maximize $F(S)$ subject to $\size{S} \leq k$.

\paragraph{Experimental setup} We compare the sparsifier approach with the mini-batch approach for each of the above datasets as follows. We first compute the $p_i$'s and fix a parameter $\beta \in (0,1)$. We take $\alpha$ in the algorithm such that $\alpha \sum_{i}^N p_i = \beta N$. That is, $\beta$ is the desired fraction of \emph{data} (in our experiments the elements that make up $F$ correspondent to elements in the dataset) we wish to sample for the sparsifier / mini-batch. Note that the $p_i$'s are the same for the sparsifier and the mini-batch approach, and the only difference is whether we sample one time in the beginning (sparsifier) or for every iteration of the greedy algorithm (mini-batch). 
For example, setting $\beta = 0.01$ means that the sparsifier will sample 1\% of the data and execute the greedy algorithm, while the mini-batch algorithm will sample 1\% of the data for every iteration of the greedy algorithm. %This also means that the number of queries made for $\beta=0.01$ is 1\% of the number of queries made by the full-batch algorithm.

In practice the naive algorithm is usually augmented with the following heuristic called \emph{lazy-greedy} \cite{lazygreedy}. The algorithm initially creates a max heap of $E$ ordered according to a key $\rho(e)$, where initially $\rho(e)=F_\myemptyset(e)$. During the $j$-th iteration it repeatedly pops $e$ from the top of the heap, updates its key $\rho(e) = F_{S_j}(e)$ and inserts it back to the heap. Due to submodularity, if $e$ remains at the top of the heap after the update we know it maximizes $F_{S_j}(e)$. While this does not change the asymptotic number of oracle queries, it often leads to significant improvement in practice.  

We also consider \emph{stochastic-greedy} \citep{BuchbinderFS15, MirzasoleimanBK15}, which further reduces the number of oracle evaluations by greedily choosing the next element added to $S_j$ from a subset of $E\setminus S_j$ of size $\frac{n}{k}\log \frac{1}{\eps}$ sampled uniformly at random. 

We apply both of the above to the sparsifier / mini-batch approach and compare them against lazy-greedy and stochastic-greedy, respectively. Finally, for every mini-batch / sparsifier variant we execute we also run a baseline where $\forall i, p_i=1/N$.

In Figure~\ref{fig: comparative} we plot results for different $\beta$ values for both the sparsifier and the mini-batch approach for different values of $k$. Every point on the graph is the average of 20 executions with the same $k,\beta$. For every dataset we present three plots: (a) the relative utility (value of $F$) of the mini-batch / sparsifier approach compared to lazy-greedy, (b) the relative number of oracle evaluations excluding the preprocessing step, and (c) relative number of oracle evaluations including the preprocessing step. To allow different $\beta$ values to fit on the same plot we use a logarithmic scale in (b). On the other hand, the preprocessing step dominates in (c), therefore, we use a linear scale. We use ``u" / ``w" prefixes for uniform / weighted sampling.  

In Figure~\ref{fig: comparative_eps} (Appendix~\ref{sec: additional experiments}) we compare the sparsifier and mini-batch algorithm (augmented with the sampling scheme of stochastic-greedy) to the stochastic-greedy algorithm. As the ground sets for both Uber pickup and Discogs are rather small, we only run experiments on the image datasets.

\paragraph{Results} We observe that the mini-batch algorithm is superior to the sparsifier approach for small values of $\beta$ while using about the same number of queries. What is perhaps most surprising is that uniform sampling outperforms over weighted sampling. We provide a theoretical explanation for this in the next section.
%While our theorems show that 
%We note that even when taking preprocessing into account, our algorithm is still significantly faster than the popular lazy-greedy algorithm.

%Finally, we also augment the mini-batch / sparsifier algorithms to sub-sample $E$ and compare them with them with stochastic-greedy. As the ground sets for both the Uber pickup dataset and the Discogs are rather small, we only run experiments on the image datasets.We present the results in the appendix. Again, we observe a significant improvement in performance when using the mini-batch algorithm.

%The $y$-axis is the relative utility of the solution compared to the (full-batch) greedy algorithm. It is clear from our experiments that the mini-batch approach is superior to the sparsifier approach for small values of $\beta$. 
%We do not provide plots for the numbers of queries
%\vspace{-150pt}
\begin{figure}[htbp]
\includegraphics[width=\linewidth]{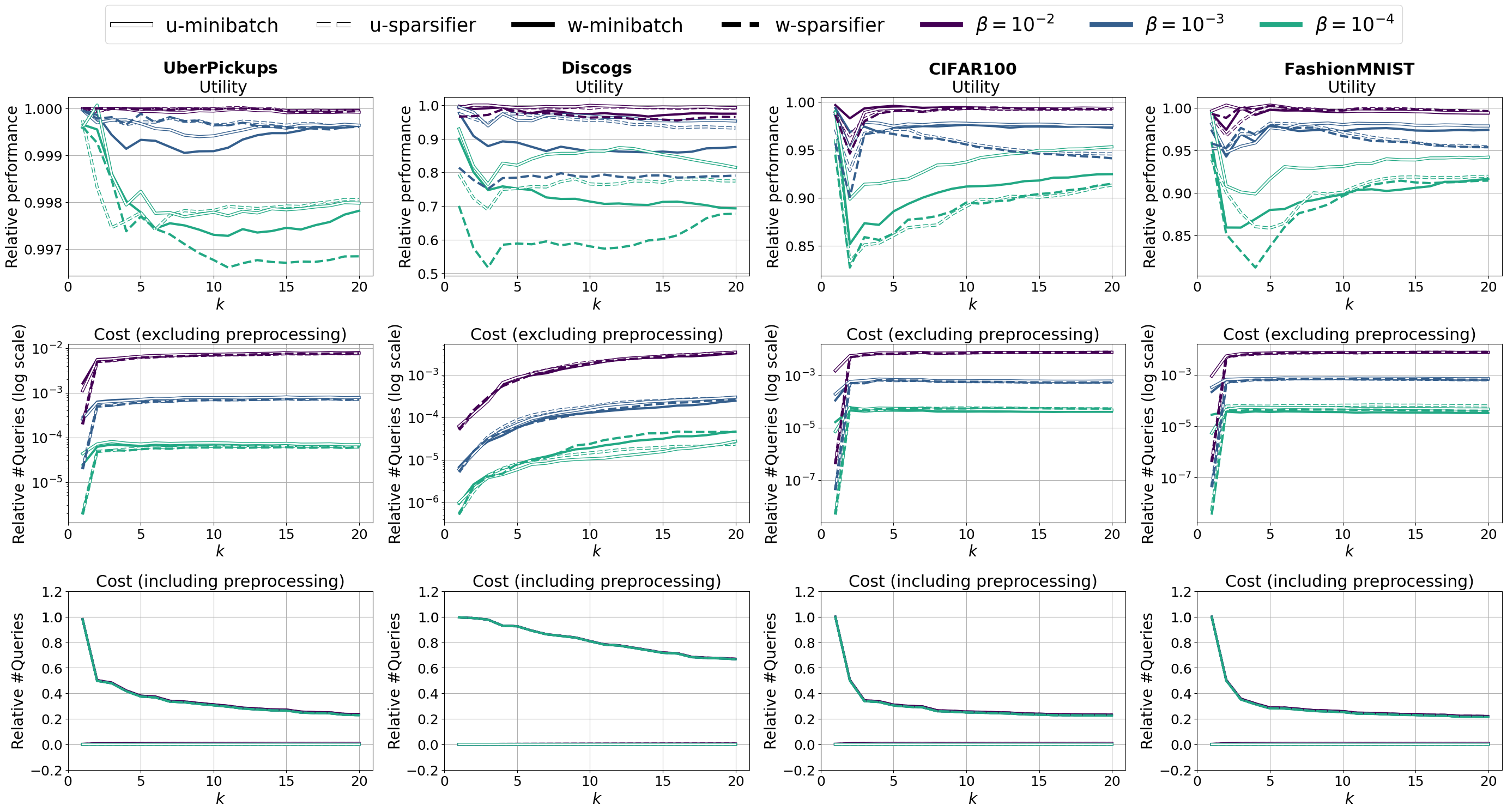}
\caption{Sparsifier and mini-batch compared with lazy-greedy.}
\label{fig: comparative}
\end{figure}

%\Gcomment{move res2 to appendix and add ref}
\iffalse
\begin{figure}[htbp]
\includegraphics[width=\linewidth]{res_eps0102.png}
\caption{Sparsifier and mini-batch compared with stochastic-greedy for $\eps=0.1$ and $\eps=0.2$.}
%\caption{Sparsifier and mini-batch compared with stochastic-greedy for $\eps=0.1$ (left) and $\eps=0.2$ (right). \Gcomment{fix. can go to appendix.}}
\label{fig: comparative_eps}
\end{figure}
\fi
\section{Smoothed analysis}
\label{sec: smoothed}
In this section we show that uniform sampling achieves better results than weighted sampling under our models of smoothing. We start by noting that under uniform sampling it holds that  $\sum_{i\in[N]}p_i = 1$ which in turn results in a bound of $\alpha$ in Lemma~\ref{lem: bound pi} instead of $n\alpha$ for the weighted case. This explains one part of the $\Theta(1/n\phi)$ factor in the transition from the ``Weighted" to ``Uniform" column in Table~\ref{table:results}. For the rest of this section we explain the $1/\phi$ factor.
%For the rest of the sec

We start by stating the following useful Chernoff bound:

%\Gcomment{Stronger version here: "The Locality of Distributed Symmetry Breaking". Can talk about chromatic number of dependency graph. Don't need distribution.}
\begin{theorem}[Bounded dependency Chernoff bound \citep{Pemmaraju01}]\label{thm:bounde chernoff}

Let $X_1,...,X_N$ be identically distributed random variables in the range $[0, 1]$ with bounded dependancy $d$. Let $X = \sum_{i=1}^{N}X_i$. Then for any $\eps \in [0, 1]$ and $\mu = \mathbb{E}[T]$ it holds that $\mathbb{P}(|X - \mathbb{E}[X]| \geq \eps \mu) \leq \Theta(d)\cdot  \exp \left(-\Theta(\eps^2 \mu/d) \right)$.
\iffalse
\begin{equation*}
\label{Chernoff_bound}
\mathbb{P}(|X - \mathbb{E}[X]| \geq \eps \mu) \leq 2 \exp \left(-\frac{\eps^2 \mu}{3a} \right)
\end{equation*}
\fi
\end{theorem}

\paragraph{Model 1} 
We show that Model 2 maintains the theoretical guarantees of both the sparsifier algorithm and the mini-batch algorithm under uniform sampling, with a multiplicative $\Theta(1/n\phi)$ factor in the query complexity.
%We formalize the transition from the ``Weighted" to ``Uniform" column in Table~\ref{table:results} in the following Theorem.

%We show that all the results in Table~\ref{table:results} hold under Model 1 with uniform sampling up to a multiplicative $O(1/\phi n)$ factor in the query complexity.
\begin{theorem}
    
\label{thm: model1}
    Assuming uniform sampling ($p_i=1/N$) and Model 1, the sparsifier construction of \cite{kraut} and our mini-batch algorithm achieve the same guarantees, with an $\Theta(1/n\phi)$ multiplicative factor in the query complexity.
    %Under Model 1 it holds w.h.p that $\forall e\in E, F(e) \geq 8\phi N$.
\end{theorem}
\begin{proof}
    
\iffalse
We start by showing that the expected value of every $f^i_\myemptyset(e)$ can be lower bounded.
Let $\delta\in [0,1]$ be some parameter. It holds that
\begin{align*}
    \mathbb{E}[f_\myemptyset^i(e)] = \mathbb{E}[f_\myemptyset^i(e) \mid f_\myemptyset^i(e) \leq \delta] \mathbb{P}[f_\myemptyset^i(e) \leq \delta] + \mathbb{E}[f_\myemptyset^i(e) \mid f_\myemptyset^i(e) > \delta] \mathbb{P}[f_\myemptyset^i(e) > \delta] \geq (1-\delta\phi)\delta
\end{align*}
The above holds for every $\delta \in[0,1]$ and is maximized by setting $\delta = 1/2\phi$, which implies that $\mathbb{E}[f_\myemptyset^i(e)] \geq 1/4\phi$. 
\fi
Assumption (1) of Model 1 implies that $\mathbb{E}[F_\myemptyset(e)] =\sum_{i\in [N]} \mathbb{E}[f^i_\myemptyset(e)] \geq N\phi$.
Applying the above Chernoff bound with $\eps=1/2, \mu = \mathbb{E}[F_\myemptyset(e)]$ we get that:
\begin{align*}
%\label{Chernoff_bound_epsilon_sparsifier}
 \mathbb{P}\left[|F_\myemptyset(e) - \mathbb{E}[F_\myemptyset(e)]| \geq \mu/2 \right]
\leq \Theta(d)  \exp(-\Theta ( \mu/ d) ) \leq \Theta(d)  \exp(-\Theta ( N\phi/d) ) \leq 1/n^2
%\leq 2 \exp \left(-\frac{\eps^2 F_S(e)}{3 F_\myemptyset(e) } \right) \leq 2 \exp \left(-\frac{\alpha \eps^2 (1-c)}{3} \right) 
%= 2 \exp \left(-\frac{\alpha \eps^2 (1-c)}{3}\right) 
%\label{chernoff_bound_epsilon_sparsifier2}
\end{align*}
Where the last inequality is due to the fact that $N =\Omega( (d/\phi) \log (nd))$. Finally, applying a union bound over all $e\in E$ we get that w.h.p $\forall e\in E, F_\myemptyset(e) \geq N\phi/2$. 

The above implies that w.h.p $p_i = \max_e f^i_\myemptyset(e)/ F_\myemptyset(e) \leq 2/\phi N$, which implies that uniform sampling approximates the weights of weighted sampling up to a $O(1/\phi)$-factor. 
It was already noted by \citet{rafiey2022sparsification} that when given a multiplicative approximation of $\set{p_i}$ the sparsifier construction goes through with a respective multiplicative increase in the size of the sparsifier. This observation also applies to the results of \citet{kraut} and for our results (Lemma~\ref{lem: concentration lemma}). 
\end{proof}

\paragraph{Empirical validation} We can empirically evaluate $\phi$ by computing $\max_e \frac{1}{N}\sum_i f^i_\myemptyset(e)$. The empirical values are as follows: CIFAR100: 0.34, FashionMNIST: 0.31, Uber pickup: $2.36 \times 10^{-4}$, Discogs: $4.10 \times 10^{-6}$. While Model 1 manages to explain the experimental results for some of the datasets, it appears that the model assumptions are a bit too strong.

%If Model 1 indeed properly models real-world instances, then we expect the empirical for our datasets to 

%the above implies that $\max_i p_i$ should be sufficiently small. We calculate these values for our datasets. CIFAR100: 0.00005787 (implies $1/\phi \geq 5.787$), FashionMNIST: 0.00006374 (implies $1/\phi \geq 6.374$), Uber pickup: 0.00650280 (implies $1/\phi \geq 8485$), Discogs: 1 (implies $1/\phi \rightarrow \infty$). While the above model manages to explain the experimental results for some of the datasets, it appears that our assumptions are a bit too strong. 

\paragraph{Model 2} We show that Model 2 maintains the theoretical guarantees of the mini-batch algorithm under uniform sampling, with a multiplicative $\Theta(1/n\phi)$ factor in the query complexity.
%We now introduce a second model with weaker assumptions.
%The above is equivalent to uniform sampling with $\alpha$ multiplied by a $\phi$ factor. This immediately translates to the sparsifier case as well.

\iffalse
When sampling uniformly we need to reweigh every element as $w_i=N/\alpha$, where $\alpha$ is the expected batch size. This implies that $w_i f^i_S(e) \leq f^i_\myemptyset(e)N/\alpha \leq 8\phi F_\myemptyset(e)/\alpha$.
Plugging this back into the relevant inequality:
\begin{align*}
\mathbb{P}\left[|\hat{F}_S(e) - F_S(e)| \geq \eps \mu \right]
\leq 2  \left(-\frac{\eps^2 \mu}{ 3a} \right) 
 \leq 2 \exp \left(-\frac{\eps^2 \alpha \mu }{24\phi F_\myemptyset(e) } \right) 
\end{align*}
We immediately get the same results as for the weighted sampling case.
\fi

%\paragraph{Smoothing model 2} We make the following change to Model 1: we assume that the smoothing assumption only holds for a \emph{single} $e$. That is, $\exists e\in E$ s.t. the probability density function from which $ f^i_\myemptyset(e)$ are drown is bounded by $\phi$.

%While the above model is much weaker than Model 1 we show that our mini-batch results still hold. The proof makes use of the greedy nature of the algorithm, and it is unclear if we can provide similar guarantees for sparsifiers under this model.

\begin{lemma}
Under Model 2, item (2) of Theorem~\ref{thm:final res} still holds with a $\Theta(1/n\phi)$ multiplicative factor in the query complexity for the mini-batch algorithm with uniform sampling.
\end{lemma}
\begin{proof}
Similar to Theorem~\ref{thm: model1} it holds w.h.p that
$F_\myemptyset(e)\geq N\phi/2$, but now this is only guaranteed for a single element $e^*\in E$.
We can use this fact to lower bound the optimum solution w.h.p: $F_\myemptyset(S^*) \geq F_\myemptyset(e^*) \geq N\phi/2$. 

Recall $a = \max \{w_if^i_S(e)\}_{i\in J}$ from Lemma~\ref{lem: concentration lemma}, where $w_i = 1 / \alpha_i$ and $\alpha_i = \min\{1, \alpha  p_i$\}. Under uniform sampling it holds that $p_i=1/N$. We conclude that $a \leq N/\alpha \leq 2F_\myemptyset(e^*)/\phi\alpha$. We get an analogue to Lemma~\ref{lem: concentration lemma}:
\begin{align*}\mathbb{P}\left[|\hat{F}_S(e) - F_S(e)| \geq \eps \mu \right]
\leq 2  \exp \left(-\frac{\eps^2 \mu}{ 3a} \right) 
 \leq 2 \exp \left(-\frac{\eps^2 \alpha \phi \mu}{6 F_\myemptyset(e^*) } \right) 
\end{align*}
 
Plugging this into the second case (non-bounded curvature) of Theorem~\ref{thm:minibatch approx guarantees} we get that the same result still holds with an additional multiplicative $1/\phi$ factor in $\alpha$.
\end{proof}
\iffalse
\begin{theorem}
    Assuming uniform sampling ($p_i=1/N$) and Model 1, the sparsifier constructions of \cite{rafiey2022sparsification, standa2023,kraut} and our mini-batch algorithm achieve the same guarantees, with only a $O(\phi)$-factor multiplicative increase in the query complexity. Under Model 2, our mini-batch algorithm achieve the same guarantees, with only a $O(k\phi)$-factor multiplicative increase in the query complexity. If we only assume $d$-dependency between $\set{f^i_\myemptyset}$ both of the above results hold with an additional $O(d)$ multiplicative factor in the query complexity.
\end{theorem}
\fi

\paragraph{Empirical validation} Under Model 2 we can empirically evaluate $\phi$ by computing $\min_e \frac{1}{N}\sum_i f^i_\myemptyset(e)$. The empirical values are as follows: CIFAR100: 0.38, FashionMNIST: 0.35, Uber pickup: $0.61$, Discogs: $0.13$. We conclude that $\phi=\Theta(1)$ and indeed Model 2 is able to explain the empirical success of the uniform mini-batch algorithm on all datasets.

\paragraph{Discussion} We observe that Model 2 manages to explain why uniform sampling outperforms in our experimental results. The empirical $\phi$ values are constant with respect to $n$. This means that the query complexity is actually less by about a $\Theta(n)$ factor for the uniform sampling case. We conclude that, given its speed and simplicity, the uniform mini-batch algorithm should be the first choice when tackling massive real-world datasets. 
\newpage

\bibliography{paper}
\bibliographystyle{iclr2025_conference}

\appendix
\section{Sparsifier background and related work}
\label{sec: sparsifiers}
\citet{rafiey2022sparsification} were the first to showe how to construct a \emph{sparsifier} for $F$. That is, given a parameter $\eps >0$ they show how to find a vector $w\in \mathbb{R}^N$ such that the number of non-zero elements in $w$ is small in expectation and the function $\hat{F} = \sum_{i=1}^N w_i f^i$ satisfies with high probability (w.h.p)\footnote{Probability at least $1-1/n^c$ for an arbitrary constant $c>1$. The value of $c$ does affect the asymptotics of the results we state (including our own).} that $\forall S \subseteq E, (1-\eps)F(S) \leq \hat{F}(S) \leq (1+\eps) F(S)$.

\iffalse
\[
\forall S \subseteq E, (1-\eps)F(S) \leq \hat{F}(S) \leq (1+\eps) F(S)
\]
\fi
Specifically, every $f^i$ is sampled with probability $\alpha_i$ proportional to $p_i=\max_{S\subseteq E, F(S)\neq 0} \frac{f^i(S)}{F(S)}$. If it is sampled, it is included in the sparsifier with weight $1/\alpha_i$, which implies that $\Expc{w_i}=1$. While calculating the $p_i$'s exactly requires exponential time, \citet{rafiey2022sparsification} make do with an approximation, which can be calculated using interior point methods \citep{BaiIWB16}.

\citet{rafiey2022sparsification} show that if all $f^i$'s are non-negative and monotone\footnote{\citet{rafiey2022sparsification} also present results for non-monotone functions, however, \citet{standa2023} point out an error in their calculation and note that the results only hold when all $f^i$'s are monotone.}, the above sparsifier can be constructed by an algorithm that requires $poly(N)$ oracle evaluations and the sparsifier will have expected size $O(\eps^{-2}Bn^{2.5} \log n)$, where $B = \max_{i\in[N]} B_i$ and $B_i$ is the number of extreme points in the base polyhedron of $f^i$. They extend their results to matroid constraints of rank $r$ and show that a sparsifier with expected size $O(\eps^{-2}Brn^{1.5} \log n)$ can be constructed. 

For the specific case of a cardinality constraint $k$, this implies a sparsifier of expected size $O(\eps^{-2} Bkn^{1.5}\log n)$ can be constructed using $poly(N)$ oracle evaluations. The sparsifier construction is treated as a \emph{preprocessing step}, and therefore the actual execution of \greedy on the sparsifier requires only $O(\eps^{-2} Bk^2n^{2.5}\log n)$ oracle evaluations to get a $(1-1/e-\eps)$ approximation. This is an improvement over \greedy when $N \gg n,B$. 

Recently, \citet{standa2023} showed improved results for the case of \emph{bounded curvature}.
The \emph{curvature} of a submodular function $F$ is defined as $c = 1-\min_{S \subseteq E, e\in E \setminus S} \frac{F_S(e)}{F_{\myemptyset}(e)}$. We say that $F$ has \emph{bounded-curvature} if $c < 1$. Submodular functions with bounded curvature \citep{ConfortiC84} offer a balance between modularity and submodularity, capturing the essence of diminishing returns without being too extreme. %The bounded curvature ensures that the function's behavior is not too far from being linear, which is extremely beneficial for optimization algorithms.

They show that when the curvature of all $f^i$'s and of $F$ is constant it is possible to reduce the preprocessing time to $O(Nn)$ oracle queries and to reduce the size of the sparsifier by a factor of $\sqrt{n}$. Furthermore, their results extend to the much more general case of \emph{$k$-submodular functions}.
While this significantly improves over the number of oracle calls compared to \citep{rafiey2022sparsification}, the runtime of the preprocessing step depends on $\log \left( \max_{i\in [N]}\frac{\max_{e\in E} f^i_\myemptyset(e)}{\min_{e\in E, f^i_\myemptyset(e)>0} f^i_\myemptyset(e)} \right)$.

Note that the $B$ factor in \cite{standa2023,rafiey2022sparsification} can be exponential in $n$ in the worst-case.

The current state of the art is due to \citet{kraut} where they show that by sampling according to $p_i=\max_{e \in E, F(e)\neq 0} \frac{f^i(e)}{F(e)}$ it is possible to get both a fast sparsifier construction time of $O(Nn)$ and a small sparsifier size of $O(\frac{n^3}{\eps^2})$. Their analysis also implies that if the solution size is bounded by $k$ (e.g., a cardinality constraint) a sparsifier of size $O(\frac{nk^2}{\eps^2})$ is sufficient. They also present results for general submodular functions, however we only focus on their results for monotone functions which are relevant for this paper.

\section{$p$-systems}
\label{sec: psys}
\paragraph{$p$-systems} The concept of \( p \)-systems offers a generalized framework for understanding independence families, parameterized by an integer \( p \). We can define a \( p \)-system in the context of an independence family \( \mathcal{I} \subseteq 2^E\) and \( E' \subseteq E \). Let \( \mathcal{B}(E') \) be the maximal independent sets within \( \mathcal{I}  \) that are also subsets of \( E' \). Formally, $\mathcal{B}(E') = \{A \in \mathcal{I} | A \subseteq E' \text{ and no } A' \in \mathcal{I} \text{ exists s.t } A \subset A' \subseteq E' \}$.
\iffalse
\[ 
\mathcal{B}(E') = \{A \in \mathcal{I} | A \subseteq E' \text{ and no } A' \in \mathcal{I} \text{ exists s.t } A \subset A' \subseteq E' \}
\] 
\fi
A distinguishing characteristic of a \( p \)-system is that for every \( E' \subseteq E \), the ratio of the sizes of the largest to the smallest sets in \( \mathcal{B}(E') \) does not exceed \( p \): $\frac{\max_{A \in \mathcal{B}(E')} |A|}{\min_{A \in \mathcal{B}(E')} |A|} \leq p$.
\iffalse
\[ 
\frac{\max_{A \in \mathcal{B}(E')} |A|}{\min_{A \in \mathcal{B}(E')} |A|} \leq p
\] 
\fi

The significance of \( p \)-systems lies in their ability to encapsulate a variety of combinatorial structures. For instance, when the intersection of \( p \) matroids can be described using \( p \)-systems. In graph theory, the collection of matchings in a standard graph can be viewed as a 2-system. Extending this to hypergraphs, where edges might have cardinalities up to \( p \), the set of matchings therein can be viewed as a \( p \)-system. 
%\Gcomment{More examples? Concrete problems?}

\paragraph{The greedy algorithm for $p$-systems} Formally, the optimization problem can be expressed as:
$\max_{S \in \mathcal{I}} F(S)$ where the pair \( (E, \mathcal{I}) \) characterizes a \( p \)-system and \( F : 2^E \rightarrow \mathbb{R}^+ \) denotes a non-negative monotone submodular set function. 
It was shown by \citet{NemhauserWF78} that the natural greedy approach achieves an optimal approximation ratio of \( \frac{1}{p + 1} \). Setting $A_j = \set{e \mid S_j +e \in \mathcal{I}}$ (i.e., $S_j$ remains an independent set after adding $e$) in Algorithm~\ref{alg: meta minibatch greedy} we get the greedy algorithm of \citet{NemhauserWF78}. Note that for general $p$-systems it might be that $k=n$, however, there are very natural problems where $k \ll n$. For example, for maximum matching $E$ corresponds to all edges in the graph, which can be quadratic in the number of nodes, while the solution is at most linear in the number of nodes.  

\section{Missing proofs} 
\paragraph{Proof of Lemma~\ref{lem: bound pi}}     We start by showing that $\sum_{i=1}^N p_i \leq n$.
    Let us divide the range $[N]$ into 
    $$A_e = \set{i\in N \mid e = \argmax_{e'\in E, F_\myemptyset(e')\neq 0} \frac{f^i_\myemptyset(e')}{F_\myemptyset(e')}}$$ If 2 elements in $E$ achieve the maximum value for some $i$, we assign it to a single $A_e$ arbitrarily. 
\begin{align*}
\sum_{i=1}^N p_i = \sum_{i=1}^N \max_{e\in E} \frac{f^i_\myemptyset(e)}{F_\myemptyset(e)} = 
\sum_{e\in E} \sum_{i\in A_e} \frac{f^i_\myemptyset(e)}{F_\myemptyset(e)} =  \sum_{e\in E} \frac{\sum_{i\in A_e} f^i_\myemptyset(e)}{F_\myemptyset(e)} \leq \sum_{e\in E} 1 \leq n    
\end{align*}
Let $X_i$ be an indicator variable for the event $w_i > 0$.
    We are interested in $\sum_{i=1}^N \mathbb{E}[X_i]$. It holds that:
    \[
     \sum_{i=1}^N \mathbb{E}[X_i] = \sum_{i=1}^N \alpha_i \leq \sum_{i=1}^N \alpha p_i = \alpha \sum_{i=1}^N p_i \leq \alpha n 
     %O(\frac{\log n}{(1-c)\eps^2} \sum_{i=1}^N{p_i}) = O(\frac{n\log n}{(1-c)\eps^2})
    \]

\paragraph{Proof of Theorem~\ref{thm:meta approx guarantees}}
\label{apx: proofs}
Theorem~\ref{thm:meta approx guarantees} directly follows from the two lemmas below. 

\begin{lemma}
\label{lem: unbounded curv approx}
     Let $\eps' \leq \eps/2k$. Algorithm~\ref{alg: meta minibatch greedy} with an additive $\eps'$-approximate incremental oracle achieves a $(1-1/e-\eps)$-approximation under a cardinality constraint $k$.
\end{lemma}
\begin{proof}
Let $S^*$ be some optimal solution for $F$.
We start by proving that the following holds for every $j\in[k]$:
 \[
    F(S_{j+1}) - F_(S_{j}) \geq  \frac{1}{k}((1-\eps)F(S^*) - F(S_j))
    \]

Fix some $j\in [k]$ and let $S^* \setminus S_j = \set{e^*_1,\dots, e^*_\ell}$ where $\ell \leq k$. Let $S^*_t = \set{e^*_1,\dots, e^*_t} $, and $S^*_0 = \emptyset$. 
    %In what follows, we use the fact that for any $S\subseteq E$ it holds that $F(S) \leq F(S^*)$ and Lemma~\ref{lem: new sparsifier}. 
    Let us first use submodularity and monotonicity to upper bound $F(S^*)$. 
    \begin{align*}
        &F(S^*) \leq F(S^* + S_j) 
        \\&= F(S_j) + \sum_{t=1}^{\ell}[ F(S_j + S^*_t)- F(S_j + S^*_{t-1})]
        \\ &\leq  F(S_j) + \sum_{t=1}^{\ell} F_{S_j}( e^*_t) \leq F(S_j) + \sum_{t=1}^{\ell} \max_{e\in E\setminus S_j} F_{S_j}(e) 
        \\ & \leq F(S_j) + k \max_{e\in E\setminus S_j} F_{S_j}(e) \\ &\leq F(S_j) + k (\max_{e\in E \setminus S_j} \hat{F}^j_{S_j}(e) + \eps' F(S^*))
    \end{align*}
    Where the last inequality is due to the fact that $\hat{F}^j_{S_j}$ is an additive $\eps'$-approximate incremental oracle.
    
    Noting that $e_j = \argmax_{e\in E \setminus S_{j}} \hat{F}^j_{S_j}(e)$ we get that:
    %Noting that $\hat{F}^j(S_{j+1}) = \max_{e\in E \setminus S_{j}} \hat{F}^j(S_j +e_j)$ we get that:    
    \begin{align*} 
    &F(S^*) \leq F(S_j) + k(\hat{F}^j_{S_j}(e_j) + \eps' F(S^*))\\
    & \implies \hat{F}^j_{S_j}(e_j) \geq \frac{1}{k}((1-\eps' k)F(S^*) - F(S_j))
    \end{align*}
    %\begin{align*} 
    %&F(S^*) \leq F(S_j) + k(\hat{F}^j(S_{j+1}) - \hat{F}^j(S_j) + \eps' F(S^*)))\\
    %& \implies \hat{F}^j(S_{j+1}) - \hat{F}^j(S_{j}) \geq \frac{1}{k}((1-\eps' k)F(S^*) - F(S_j))
    %\end{align*}
    The above lower bounds the progress on the $j$-th mini-batch. Now, let us bound the progress on $F$. Again, we use the fact that $\hat{F}^j_{S_j}$ is an additive $\eps'$-approximate incremental oracle.
    \begin{align*} 
    &F(S_{j+1}) - F(S_{j}) \geq \hat{F}^j_{S_j}(e_j) - \eps' F(S^*)  
    \\ &\geq \frac{1}{k}((1-\eps' k)F(S^*) - F(S_j)) -  \eps' F(S^*) \\ &\geq \frac{1}{k}((1-2\eps' k)F(S^*) - F(S_j))
    \end{align*}

    Finally, using the fact that $\eps' \leq \eps/2k$ we get:
    \[
    F(S_{j+1}) - F(S_{j}) \geq  \frac{1}{k}((1-\eps)F(S^*) - F(S_j))
    \]

Rearranging, the result directly follows using standard arguments.
\begin{align*}
        &F(S_{k+1}) > \frac{(1-\eps)}{k}F(S^*) + (1- \frac{1}{k })F(S_{k}) 
        \\ &\geq \frac{(1-\eps)}{k}F(S^*)(\sum_{i=0}^k (1-\frac{1}{k})^i) + F(\emptyset)
    \\ &\geq F(S^*)\frac{(1-\eps)(1-\frac{1}{k})^k}{k(1-(1-\frac{1}{k}))} = (1-\eps)(1- \frac{1}{k})^k F(S^*) 
    \\& \geq (1-\eps)(1-1/e) F(S^*) \geq (1-1/e-\eps) F(S^*)
\end{align*}

\end{proof}

\begin{lemma}
\label{lem: psys approx}
     Let $\eps' \leq \eps/2kp$. Algorithm~\ref{alg: meta minibatch greedy} with an additive $\eps'$-approximate incremental oracle achieves a $(\frac{1-\eps}{1+p})$-approximation under a $p$-system constraint.
\end{lemma}
\begin{proof}
    Let $S^*$ be some optimal solution for $F$.
Assume without loss of generality that the solution returned by the algorithm consists of $k$ elements $S_{k+1} = \set{e_1,\dots,e_k}$.
%The greedy algorithm combined with Corollary~\ref{col: additive error} guarantees that $\hat{F}^j_{S_{j}}(e_j) \geq \max_{e\in A_j} F_{S_{j}}(e) -2\eps' F(S^*)$. 

%Let 
%\[
%F_{S_{j}}(e_j) = F_{S_{j}}(e_j) = F(S_{j+1}) - F(S_{j}) 
%\]
 %be the improvement in the solution when $e_j$ is added. 
 %Note that $F(S_{k+1}) \geq \sum_{j=1}^k F_{S_{j}}(e_j)$. 
 We show the existence of a partition $S^*_1,S^*_2,\dots,S^*_k$ of $S^*$ such that $F_{S_{j}}(e_j)\geq \frac{1}{p}F_{S_{k+1}}(S^*_j) - 2\eps' F(S^*)$. Note, we allow some of the sets in the partition to be empty. 

Define $T_k = S^*$. For $j=k,k-1,...,2$ execute: Let $B_j = \set{e\in T_j \mid S_{j} +e\in \mathcal{I}}$. If $\size{B_j} \leq p$ set $S^*_j = B_j$; else pick an arbitrary $S^*_j \subset B_j$ with $\size{S^*_j}=p$. Then set $T_{j-1} = T_j \setminus S^*_j$ before decreasing $j$. After the loop set $S^*_1 = T_1$. It is clear that for $j=2,...,k, \size{S^*_j} \leq p$. 

We prove by induction over $j=0,1,...,k-1$ that $\size{T_{k-j}}\leq (k-j)p$. For $j=0$, when the greedy algorithm stops, $S_{k+1}$ is a maximal independent set contained in $E$, therefore any independent set (including $T_k = S^*$) satisfies $\size{T_k} \leq p\size{S_{k+1}} = pk$. We proceed to the inductive step for $j>0$. There are two cases: (1) $\size{B_{k-j+1}} >p$, which implies that $\size{S^*_{k-j+1}} = p$ and using the induction hypothesis we get that $\size{T_{k-j}}=\size{T_{k-j+1}} - \size{S^*_{k-j+1}} \leq (k-j+1)p-p=(k-j)p$. (2) $\size{B_{k-j+1}} \leq p$, it holds that $T_{k-j} = T_{k-j+1}\setminus B_{k-j+1}$. Let $Y=S_{k-j+1}+ T_{k-j}$. Due to the definition of $B_{k-j+1}$ it holds that $S_{k-j+1}$ is a maximal independent set in $Y$. It holds that $T_{k-j}$ is independent and contained in $Y$, therefore $\size{T_{k-j}} \leq p\size{S_{k-j+1}} = p(k-j)$.

Finally, we get that $\size{T_1} = \size{S^*_1} \leq p$. By construction it holds that $\forall j\in [k], \forall e\in S^*_{j}, S_{j}+e\text{ is independent}$. From the choice made by the greedy algorithm and the fact that $\hat{F}^j_{S_j}$ is an additive $\eps'$-approximate incremental oracle it follows that for each $e\in S^*_j$:
\begin{align*}
  &F_{S_j}(e_j) \geq \hat{F}^j_{S_{j}}(e_j) - \eps' F(S^*)   
  \\ &\geq \hat{F}^j_{S_{j}}(e) - \eps' F(S^*) \geq F_{S_{j}}(e) - 2\eps' F(S^*)  
\end{align*}

Hence,
\begin{align*}
 &\size{S^*_j}F_{S_{j}}(e_j) \geq \sum_{e\in S^*_j} (F_{S_{j}}(e) - 2\eps' F(S^*)) 
\\ & \geq F_{S_j}(S^*_j) - 2\eps' \size{S^*_j} F(S^*) \geq F_{S_{k+1}}(S^*_j) - 2\eps' \size{S^*_j} F(S^*)    
\end{align*}

Using submodularity in the last two inequalities.

For all $j\in \set{1,2,...,k}$ it holds that $\size{S^*_j} \leq p$, and thus $F_{S_{j}}(e_j) \geq \frac{1}{p} F_{S_{k+1}}(S^*_j) - 2\eps' F(S^*)$. Using the partition we get that:
\begin{align*}
& F(S_{k+1})\geq \sum_{j=1}^k F_{S_{j}}(e_j) \geq \sum_{j=1}^k (\frac{1}{p}F_{S_{k+1}}(S^*_j) - 2\eps' F(S^*))
\\ &\geq \frac{1}{p}F_{S_{k+1}}(S^*)- 2\eps' kF(S^*) 
\\ & \geq \frac{1}{p}(F(S^*) - F(S_{k+1})) - 2\eps' kF(S^*)    
\end{align*}
Where the second to last inequality is due to submodularity and the last is due to monotonicity.
Rearranging we get that:
\[
F(S_{k+1}) \geq \frac{(1-2p\eps' k)}{p+1} F(S^*)
\]
As $\eps' < \frac{\eps}{2pk}$ we get the desired result. 
\end{proof}

\section{Additional experiments}
\label{sec: additional experiments}
\begin{figure}[htbp]
\includegraphics[width=\linewidth]{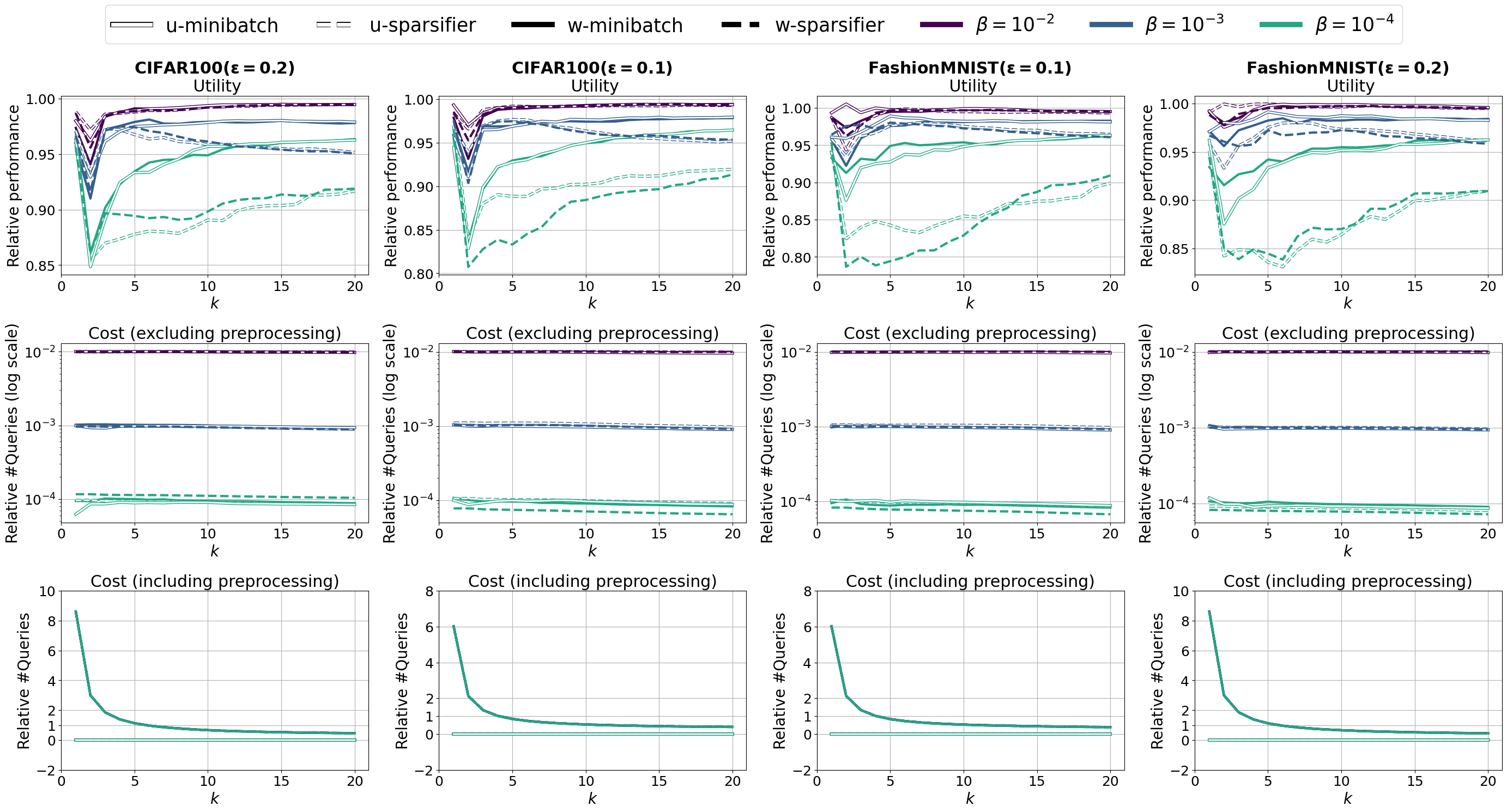}
\caption{Sparsifier and mini-batch compared with stochastic-greedy for $\eps=0.1$ and $\eps=0.2$.}
%\caption{Sparsifier and mini-batch compared with stochastic-greedy for $\eps=0.1$ (left) and $\eps=0.2$ (right). \Gcomment{fix. can go to appendix.}}
\label{fig: comparative_eps}
\end{figure}

%\section{Additional experiments} 
%\paragraph{}
%\bibliographystyle{alpha}
%\bibliography{paper}

\end{document}